\documentclass[11pt]{article}

\usepackage[preprint]{neurips_2026}

\usepackage{amsmath,amssymb,amsthm,mathtools}
\usepackage{bm}
\usepackage{algorithm}
\usepackage{algorithmic}
\usepackage{enumitem}
\usepackage{booktabs}       
\usepackage{xcolor}         
\usepackage{url}
\usepackage[normalem]{ulem} 

\theoremstyle{definition}

\newtheorem{assumption}{Assumption}[section]

\theoremstyle{plain}
\newtheorem{theorem}{Theorem}[section]

\newtheorem{proposition}{Proposition}[section]
\newtheorem{corollary}{Corollary}[section]

\newcommand{\R}{\mathbb{R}}

\newcommand{\dd}{\,\mathrm{d}}
\newcommand{\norm}[1]{\left\lVert #1\right\rVert}
\newcommand{\ip}[2]{\left\langle #1,#2\right\rangle}
\newcommand{\grad}{\nabla}
\newcommand{\divg}{\mathrm{div} \; }

\newcommand{\I}{\mathbf{I}}

\title{PathGuide: Dynamic Classifier-Free Guidance via On-Policy Transport Alignment}

\author{%
  Avishag Nevo\thanks{Correspondence to \texttt{avishag.nevo@campus.technion.ac.il}.} \\
  Technion - Israel Institute of Technology \\
  \texttt{avishag.nevo@campus.technion.ac.il} \\
  \And
  Tamir Hazan \\
  Technion - Israel Institute of Technology \\
  \texttt{tamir.hazan@technion.ac.il} \\
}

\begin{document}
\maketitle

\begin{abstract}
While modern generative models excel at modeling complex data, precise inference-time control in conditional generation remains a critical challenge. Classifier-free guidance (CFG) is a primary mechanism for such control, yet it is typically treated as a static tuning parameter. In flow-based models, however, the guidance scale fundamentally dictates the velocity field and the resulting probability path, making guidance selection a dynamic path-optimization problem. We introduce \emph{PathGuide}, a framework that reformulates scalar CFG selection as an on-policy transport problem. Leveraging the \emph{weak form of the continuity equation}, we derive a selection criterion with a direct path-correctness interpretation: we prove that if the guided field is weakly equivalent to the exact conditional field along the generated rollout, the sampler's path coincides with the target conditional law. For scalar CFG, this criterion yields a strictly quadratic local objective with an efficient, closed-form selector for each solver interval. PathGuide enables optimal guidance scales to be computed and used \emph{online} during generation or fitted offline as a \emph{reusable} piecewise-constant schedule. We validate our method on low-resolution image manifolds and controlled settings across various continuous-time flow constructions, demonstrating that this transport-based selector improves path alignment and sample fidelity over both fixed and state-of-the-art adaptive guidance baselines.
\end{abstract}

\section{Introduction}
Modern generative modeling has achieved remarkable success due to its ability to model complex data distributions with minimal discrepancy between real and generated data laws. This is largely achieved by learning continuous-time dynamics \citep{dhariwal2021diffusionbeatgans,ramesh2022hierarchicalcliplatents,rombach2022latentdiffusion,saharia2022photorealistic,betker2023bettercaptions,esser2024scalingrectifiedflow}. Within this paradigm, score-based models learn reverse stochastic or probability-flow dynamics \citep{sohldickstein2015nonequilibrium,ho2020ddpm,song2021scorebasedsde,song2021ddim,karras2022edm}, while flow-based models learn probability paths via velocity-field regression \citep{lipman2023flowmatching,liu2022flowstraightfastlearning,tong2024minibatchot,albergo2025stochasticinterpolants,huang2026improvingflowmatchingaligning,chen2024flowmatchinggeneralgeometries, kapusniak2024metricflowmatching}.

Conditional generative modeling has enabled a vast array of machine learning applications, including text-to-image synthesis \citep{ramesh2022hierarchicalcliplatents,rombach2022latentdiffusion,saharia2022photorealistic,podell2023sdxl,dai2023emu}, video generation \citep{blattmann2023videoldm,blattmann2023stablevideodiffusion,singer2022makeavideo}, audio synthesis \citep{evans2024fasttiminglatent,wang2023neuralcodec,le2023voicebox}, robotics and decision-making \citep{chi2023diffusionpolicy,chen2021decisiontransformer,janner2021offline,janner2022diffuser,ajay2022decisionmaking}, biomolecular modeling \citep{abramson2024alphafold3,corso2023diffdock,yim2023se3diffusion}, and post-training control tasks such as editing, inpainting, and spatial motion constraints \citep{avrahami2022blendeddiffusion,brooks2023instructpix2pix,meng2022sdedit,lugmayr2022repaint,zhang2023controlnet,bar-tal2023multidiffusion,garipov2023compositional,watanabe2026projflow}.

In these settings, inference-time control is paramount. Once a large-scale generator is trained, it is typically steered toward a desired condition using Classifier-Free Guidance (CFG), which usually relies on a constant scale or a hand-crafted schedule \citep{dhariwal2021diffusionbeatgans,ho2022classifierfree,ho2021classifierfreeworkshop, zheng2023guidedflows,lipman2024flowmatchingguidecode}. While effective at improving conditional alignment, misconceptions regarding the resulting sample law persist \citep{bradley2024classifierfreepredictorcorrector, chidambaram2024whatdoesguidancedo}, revealing deep path-wise inconsistencies: different segments of the trajectory often require varying guidance strengths \citep{kynkaanniemi2024guidanceinterval,wang2024cfgschedulers,castillo2023adaptiveguidance,jin2026stagewisecfg}. Excessive guidance scales can reduce diversity, trigger mode collapse and push samples away from the learned manifold \citep{chung2025cfgplusplus,jin2025angledomainguidance,wang2025goldencfgpath,saini2025rectifiedcfgflowbasedmodels,cai2026cfgmp, fan2025cfgzerostar,sadat2025eliminatingoversaturation, he2024manifoldpreservingguided, chidambaram2024whatdoesguidancedo}. Consequently, guidance selection should be treated as a path-optimization problem rather than a simple hyperparameter tuning exercise.

Existing methods attempt to mitigate these failures by heuristically adapting the guidance scale, training auxiliary networks for path correction, or imposing geometric constraints to counteract faulty interactions between the guidance scale and the model \citep{kynkaanniemi2024guidanceinterval,wang2024cfgschedulers,sadat2023cads,xia2024timesteptuner,yehezkel2025annealingguidance,galashov2026learntoguide,koulischer2025feedbackguidance,chung2025cfgplusplus,jin2025angledomainguidance,fan2025cfgzerostar,wang2025goldencfgpath,saini2025rectifiedcfgflowbasedmodels,cai2026cfgmp,watanabe2026projflow,sadat2025eliminatingoversaturation,he2024manifoldpreservingguided, guo2024gradientguidance}. While useful, these approaches do not address a fundamental transport problem: which scalar CFG value ensures that the sampler's actual rollout remains faithful to the exact conditional probability path.

We introduce PathGuide, a framework that reformulates scalar CFG selection as an on-policy transport problem. Rather than treating guidance as a heuristic scaling factor, we derive a principled criterion using the weak form of the continuity equation \citep{villani2009optimal, ambrosio2005gradient,santambrogio2015optimaltf} to establish a theoretical foundation for path correctness. Specifically, we derive a local objective with a direct path-correctness interpretation: if the guided field is weakly equivalent to the exact conditional field under the generated rollout, then - under the uniqueness of the weak continuity equation \citep{DiPerna1989, Ambrosio2004} - the generated rollout coincides with the exact conditional probability path. By evaluating this objective directly on the sampler’s empirical rollout and leveraging the endpoint-conditioned structure of flow matching \citep{lipman2023flowmatching,tong2024minibatchot,albergo2025stochasticinterpolants}, we translate this theory into a practical framework through the following contributions:

\begin{itemize}[leftmargin=1.5em]
\item We derive a rigorous local objective providing a principled measure of compatibility between the sampler’s realized rollout and the target conditional law at each time step.
\item  We prove that for the scalar CFG family, this objective is \emph{strictly quadratic} and admits an efficient closed-form for the optimal guidance scale.
\item  We introduce a practical algorithm that supports both \emph{real-time online adaptation} during sampling and the offline generation of \emph{reusable}, piecewise-constant guidance schedules.
\item Using controlled Gaussian-mixture flow-matching experimental validation across multiple flow families, we demonstrate consistent improvements in path alignment and endpoint fidelity over existing training-free baselines under matched protocols.
\end{itemize}

\section{Related Work}
The emergence of Flow Matching and Rectified Flow has shifted the generative modeling paradigm from discrete diffusion steps to continuous-time transport along learned velocity fields \citep{lipman2023flowmatching, liu2022flowstraightfastlearning, tong2024minibatchot, albergo2025stochasticinterpolants}. While these works utilize the decomposition of marginal dynamics into endpoint-conditioned paths as a training objective, we repurpose this decomposition as an inference-time diagnostic tool to evaluate whether a guided velocity field remains locally consistent with the ground-truth transport law.

Ensuring this consistency is non-trivial because the optimal guidance scale is not a global constant. In diffusion and flow models alike, studies show that different intervals of the generative trajectory exhibit varying sensitivities to conditioning \citep{kynkaanniemi2024guidanceinterval, wang2024cfgschedulers}. While this observation has motivated adaptive, annealed, and feedback-based schedulers that modulate guidance via heuristic triggers or learned policies \citep{castillo2023adaptiveguidance, jin2026stagewisecfg, yehezkel2025annealingguidance, galashov2026learntoguide, koulischer2025feedbackguidance}, such methods often lack a formal distributional guarantee. In contrast, we derive a variational criterion directly from the weak continuity equation \citep{DiPerna1989, Ambrosio2004,ambrosio2005gradient, santambrogio2015optimaltf}, providing a first-principles derivation for guidance schedules that preserves the integrity of the probability path without the need for manual tuning or auxiliary training.

Beyond merely scheduling the guidance scale, a concurrent line of research focuses on "repairing" the classifier-free guidance update itself to mitigate known failure modes like mode collapse \citep{dhariwal2021diffusionbeatgans, ho2022classifierfree}. These manifold-aware corrections \citep{chung2025cfgplusplus, jin2025angledomainguidance,sadat2025eliminatingoversaturation, he2024manifoldpreservingguided, guo2024gradientguidance} and solver-level refinements \citep{fan2025cfgzerostar, wang2025goldencfgpath, cai2026cfgmp, saini2025rectifiedcfgflowbasedmodels} aim to stabilize the solver by optimizing the \emph{form} of the update. Our work, however, optimizes the scalar \emph{input} to that update.

This shift toward dynamic selection reflects a broader trend where generative models are evaluated not just by terminal samples, but by the alignment of the entire induced probability path \citep{huang2026improvingflowmatchingaligning, albergo2025stochasticinterpolants, watanabe2026projflow}. Our work further advances this frontier by treating guidance selection as an on-policy transport problem; unlike existing schedules defined \emph{a priori} on idealized paths, our framework optimizes the guidance scale directly on the sampler’s realized rollout to ensure the trajectory remains faithful to the target conditional distribution.

The remainder of this paper is organized as follows: Section~\ref{sec:background} establishes the foundations of continuous-time flows and guidance, Section~\ref{sec:method} introduces the PathGuide framework and our on-policy consistency objective, and Section~\ref{sec:experiments} demonstrates that our framework is robust across various continuous-time flow constructions, ranging from Optimal Transport to Variance-Preserving dynamics.

\section{Continuous-Time Flows and Guidance}
\label{sec:background}

Continuous-time generative models have recently emerged as a state-of-the-art paradigm for high-dimensional distribution modeling  \citep{chen2018neural, song2021scorebasedsde}.
We consider a target data distribution \(q(x_1)\) on \(\mathbb{R}^d\).
A flow model describes a probability path \(p_t\) that evolves over the normalized time interval \(t \in [0,1]\), such that \(p_0\) is a tractable reference distribution, e.g., a standard Gaussian, and \(p_1 \approx q\).

A time-dependent vector field (or velocity field) \(u_t: \mathbb{R}^d \to \mathbb{R}^d\), defines the dynamics of the generative process.
This field induces a flow map \(\phi_t: \mathbb{R}^d \to \mathbb{R}^d\) through the ordinary differential equation (ODE)
\begin{equation}
\frac{d}{dt} \phi_t(x) = u_t(\phi_t(x)), \qquad \phi_0(x) = x.
\label{eq:flow-map}
\end{equation}
The probability density \(p_t\) at any time \(t\) is the pushforward of the initial law under this map, denoted \(p_t = (\phi_t)_{\#} p_0\).
This relationship implies that the pair \((p_t, u_t)\) satisfies the continuity equation \citep{villani2009optimal,ambrosio2005gradient}, which describes the local conservation of probability mass:
\begin{equation}
\frac{\partial}{\partial t} p_t(x) + \nabla \cdot \bigl( p_t(x) u_t(x) \bigr) = 0.
\label{eq:continuity-equation}
\end{equation}

For a smooth compactly supported test function \(\psi\), the same equation can be written in weak form as \citep{santambrogio2015optimaltf}:
\begin{equation}
\frac{\dd}{\dd t}
\int_{\R^d}\psi(x)p_t(x)\,\dd x
=
\int_{\R^d}
\nabla\psi(x)\cdot u_t(x)\,p_t(x)\,\dd x .
\label{eq:weak-form-continuity-equation}
\end{equation}

\paragraph{Classifier-free guidance for flow models.}
In conditional generation, we aim to sample from \(q(x_1 \mid y)\) by steering the flow toward a specific condition \(y\).
Classifier-free guidance (CFG) \citep{ho2021classifierfreeworkshop} was originally developed for score-based diffusion, where a Bayes' rule decomposition of the conditional score replaces an explicitly trained classifier \citep{dhariwal2021classifierguidance}; since the same marginal path is generated by a probability-flow ODE whose field depends on the score, guiding the score induces a guided velocity field. Appendix~\ref{app:diffusion-paths} gives this equivalence in full.

In modern flow-based models, guidance is therefore usually implemented directly at the level of learned velocity fields \citep{zheng2023guidedflows}.
Let
\(
v_t^\theta(\cdot\mid \varnothing)
\)
denote the learned unconditional field, and let
\(
v_t^\theta(\cdot\mid y)
\)
denote the learned conditional field.
Classifier-free guidance forms the guided field
\begin{equation}
v_t^\theta(x\mid y;\omega_t)
=
v_t^\theta(x\mid \varnothing)
+
\omega_t
\Bigl(
v_t^\theta(x\mid y)
-
v_t^\theta(x\mid \varnothing)
\Bigr),
\label{eq:guided-velocity-net}
\end{equation}
where $\omega_t$ is a scalar guidance value. The guidance scale $\omega_t$ serves as an extrapolation parameter: $\omega_t=1$ recovers the nominal learned conditional field, while $\omega_t > 1$ is conventionally used to amplify the discrepancy between the conditional and unconditional dynamics. This inference-time interface has become a cornerstone of state-of-the-art flow systems \citep{esser2024scalingrectifiedflow, blackforestlabs2024announcingflux, sd35_stabilityai}.

\paragraph{Flow Matching.}
\label{par:flow-matching}
Flow matching \citep{lipman2023flowmatching} provides a framework to learn the conditional field $u_t(\cdot \mid y)$ by marginalizing over endpoint-conditioned paths. For a fixed sample $x_1 \sim q(x_1 \mid y)$, we define an endpoint-conditioned probability path $p_t(\cdot \mid x_1, y)$ and its corresponding vector field $u_t(\cdot \mid x_1, y)$. The exact marginal conditional field $u_t(x \mid y)$ is then defined as:
\begin{equation}
u_t(x \mid y) = \int_{\mathbb{R}^d} u_t(x \mid x_1, y) \frac{p_t(x \mid x_1, y) q(x_1 \mid y)}{p_t(x \mid y)} dx_1.
\label{eq:marginal-vector-t-cond}
\end{equation}
By construction, this marginal field is the unique vector field that generates the conditional marginal path 
\begin{equation}
p_t(x \mid y) = \int p_t(x \mid x_1, y) q(x_1 \mid y) dx_1 
\label{eq:marginal-probability-density-t-cond}
\end{equation}
and satisfies the continuity equation in \eqref{eq:continuity-equation} \citep{lipman2023flowmatching}. In practice, we approximate this exact field $u_t(x \mid y)$ with a neural network $v_t^\theta(x \mid y)$ via a regression objective, which uses the known form of $u_t(x \mid x_1, y)$ determined by the flow. At inference time, one can sample $x_t \sim p_t(x \mid y)$ using the learned vector field by sampling $x_0 \sim p_0$ and integrating: $x_t = x_0 + \int_{0}^{t} v_s^\theta(x_s \mid y) ds$. 

In the idealized setting where the vector fields are exact, i.e., if \(v_t^\theta(\cdot \mid y) \equiv u_t(\cdot \mid y)\) for every $y$, the value \(\omega_t \equiv 1\) would be sufficient to recover the exact marginal path \(p_t(\cdot \mid y)\).
However, in practice, \(v_t^\theta\) is a learned estimator, and in this regime, \(\omega_t\) can be interpreted as a corrective parameter that calibrates the guided flow to account for approximation errors in the neural network.
In this paper, we develop a principled, theoretically grounded framework for selecting the schedule \(\omega_t\) to optimally compensate for the mismatch between the frozen learned field and the true underlying transport law.

\section{Dynamic Classifier-Free Guidance via On-Policy Transport Alignment}
\label{sec:method}

The learned velocity field $v_t^\theta(x \mid y; \omega_t)$, defined in Equation~\eqref{eq:guided-velocity-net}, is an imperfect estimator of the true marginal vector field $u_t(x \mid y)$. Because the generative process is sequential, any approximation error introduced at an earlier time $s < t$ propagates through the ODE solver, causing the realized probability path $\hat{p}_t(\cdot \mid y)$ to drift away from the target conditional distribution $p_t(\cdot \mid y)$.

We propose to treat the guidance scale $\omega_t$ as a dynamic control variable that targets these accumulated learning errors. At each time, $\omega_t$ is chosen to minimize a weak-form mismatch between the guided field and the exact conditional field, evaluated \emph{under the rollout law realized so far}. This on-policy criterion is a tractable surrogate for path correctness: it is exact under the ideal conditions of Proposition~\ref{prop:path-correctness}, and Section~\ref{sec:experiments} measures what it achieves on a rollout that has already drifted. We denote the guidance history on the continuous interval $s \in [0, t)$ as $\omega_{[0, t)} = (\omega_s)_{s \in [0, t)}$. With this notation, our goal is to use $\omega_t$ to adaptively recalibrate the realized probability path, $\hat{p}_t(\cdot \mid y; \omega_{[0, t)})$, toward the target conditional distribution $p_t(\cdot \mid y)$.

\begin{proposition}[Ideal weak equivalence implies path correctness]
\label{prop:path-correctness}
Fix condition $y$. Assume that the generated rollout distribution $\hat{p}_s(\cdot \mid y; \omega_{[0,s)})$ is weakly continuous in time and satisfies the weak continuity identity with the guided field $v_s^\theta$ (per Assumption~\ref{ass:rollout-law}):
\begin{equation}
\frac{\mathrm{d}}{\mathrm{d}s}\int_{\mathbb{R}^d}\psi(x)\, \hat{p}_s(x\mid y; \omega_{[0,s)}) \,\mathrm{d}x
=
\int_{\mathbb{R}^d} \nabla\psi(x)\cdot v_s^\theta(x \mid y; \omega_s)\, \hat{p}_s(x\mid y; \omega_{[0,s)}) \,\mathrm{d}x
\end{equation}
for every $\psi \in C_c^\infty(\mathbb{R}^d)$ and almost every $s \in [0, t]$. 

Assume further that the guided field and the exact velocity field $u_s(\cdot \mid y)$ are weakly equivalent under the generated rollout:
\begin{equation}
\int_{\mathbb{R}^d} \nabla\psi(x)\cdot v_s^\theta(x \mid y; \omega_s)\, \hat{p}_s(x\mid y; \omega_{[0,s)}) \,\mathrm{d}x
=
\int_{\mathbb{R}^d} \nabla\psi(x)\cdot u_s(x \mid y)\, \hat{p}_s(x\mid y; \omega_{[0,s)}) \,\mathrm{d}x.
\end{equation}
Since the exact conditional path $p_s(\cdot \mid y)$ is the unique weak solution to the continuity equation driven by $u_s(\cdot \mid y)$ (under Assumption~\ref{ass:ce-uniqueness}) with the same initial condition $\hat{p}_0 = p_0$, then
\begin{equation}
\hat{p}_s(\cdot \mid y; \omega_{[0,s)}) = p_s(\cdot \mid y), \qquad \forall s \in [0,t].
\end{equation}
\end{proposition}

\begin{proof}[Proof sketch]
Because the generated rollout law satisfies the weak continuity identity under the guided field, and the guided field is weakly equivalent to the exact conditional field under that law, the rollout also satisfies the weak continuity equation driven by the exact conditional field. By the uniqueness of weak solutions (per Assumption~\ref{ass:ce-uniqueness}) with the same initial condition, the generated rollout path must coincide with the exact conditional path. Full proof is in Appendix~\ref{app:path-correctness}.
\end{proof}

While a theoretical continuous-time selector would choose a value at every instant $t$, a numerical solver typically utilizes a single value per integration interval. Let $0 = t_0 < t_1 < \dots < t_T = 1$ be the solver grid. We approximate the continuous guidance schedule with a step-wise sequence of scales. At each discrete step $i \in \{0, \dots, T-1\}$, given the history of previous guidance scales $\vec{\omega}_{i}(y) = (\omega_0, \dots, \omega_{i-1})$, abbreviated $\vec{\omega}_{i}$, we denote by $\hat{p}_{t_i}(\cdot \mid y; \vec{\omega}_i)$ the distribution of the samples generated up to that point.
Using the guided vector field \(v_{t_i}^\theta(x \mid y; \omega_i)\), our goal is to adaptively determine the optimal \(\omega_i\) for the interval \([t_i,t_{i+1})\). Following Proposition~\ref{prop:path-correctness}, we choose \(\omega_i\) to approximate the ideal weak equivalence condition. Applied iteratively from the first solver step, each update minimizes the local residual \emph{under the rollout realized by the already committed schedule}, so every scale is selected against the state the sampler is actually in rather than against an idealized path:
\begin{equation}
\mathcal{L}_i(\omega_i \mid y;\vec{\omega}_i,\psi)
= \left(
\int_{\mathbb{R}^d}
\Bigl(
u_{t_i}(x\mid y)-v_{t_i}^\theta(x\mid y;\omega_i)
\Bigr)\cdot
\nabla_x\psi(x)\,
\hat p_{t_i}(x\mid y; \vec{\omega}_i)\,\mathrm{d}x \right)^2.
\label{eq:exact-local-objective}
\end{equation}

\subsection{Closed-form local selector under scalar CFG}

We now specialize to scalar classifier-free guidance. At this point, the key observation is that under scalar CFG, the objective is affine in the current local control value. See Appendix~\ref{app:affine-objective}.

\begin{corollary}[Exact local selector]
\label{thm:exact-local-selector}
Fix $y$, and $\psi \in C_c^\infty(\mathbb{R}^d)$, and a committed past schedule $\vec{\omega}_i$.
Then $\mathcal{L}_i(\omega_i \mid y; \vec{\omega}_i, \psi)$, defined in Equation~\eqref{eq:exact-local-objective} is minimized for
\begin{equation}
\omega_i^\star(y; \vec{\omega}_i, \psi)
=
\frac{
\displaystyle\int_{\mathbb{R}^d}
\Bigl(
u_{t_i}(x \mid y) - v_{t_i}^\theta(x \mid \varnothing)
\Bigr) \cdot
\nabla_x\psi(x) \,
\hat{p}_{t_i}(x \mid y; \vec{\omega}_i) \,\mathrm{d}x
}{
\displaystyle\int_{\mathbb{R}^d}
\Bigl(
v_{t_i}^\theta(x \mid y) - v_{t_i}^\theta(x \mid \varnothing)
\Bigr) \cdot
\nabla_x\psi(x) \,
\hat{p}_{t_i}(x \mid y; \vec{\omega}_i) \,\mathrm{d}x
},
\label{eq:exact-local-selector}
\end{equation}
provided the denominator is nonzero. If the denominator vanishes, the objective is constant in $\omega_i$, so every admissible value is optimal.
\end{corollary}

\begin{proof}[Proof sketch]
Since the guided field is affine in the scalar guidance value \(\omega _{i}\), the local objective is a one-dimensional quadratic. Setting the derivative with respect to \(\omega _{i}\) to zero yields \(\omega _{i}^{\star }\). The full proof is provided in Appendix~\ref{app:exact-local-selector-proof}.
\end{proof}

Notably, if the learned conditional field is exact (\(u_{t_i} = v_{t_i}^\theta\)), the numerator and denominator in \eqref{eq:exact-local-selector} coincide, yielding \(\omega_i^\star = 1\) (or making any value optimal if the denominator is zero).

\subsection{Endpoint-conditioned representation of the exact objective}

The adaptive loss is written in terms of the exact marginal field \(u_t(\cdot\mid y)\). However, in flow matching the marginal field is not accessed directly; it is obtained by averaging endpoint-conditioned vector fields with posterior weights, as described in Equation~\eqref{eq:marginal-vector-t-cond}. The following equivalence is what makes the objective estimable: it removes \(u_{t_i}(\cdot\mid y)\) in favor of endpoint-conditioned quantities that the flow construction supplies in closed form.

\begin{theorem}[Endpoint-conditioned representation of the local objective]
\label{thm:endpoint-conditioned-objective}
Fix \(y\), \(\psi \in C_c^\infty(\mathbb{R}^d)\), and a committed past schedule \(\vec{\omega}_i\).
Under Assumptions~\ref{ass:conditional-path-regularity} and~\ref{ass:endpoint-posterior}, the exact
local objective \eqref{eq:exact-local-objective} admits the endpoint-conditioned form
$\mathcal{L}_i(\omega_i \mid y;\vec{\omega}_i,\psi)
=$
\begin{align}
\left( 
\int_{\R^d}
\int_{\R^d}
(
u_{t_i}(x\mid y,x_1)
-
v_{t_i}^\theta(x\mid y;\omega_{t_i})
)\, 
\frac{
p_{t_i}(x\mid y,x_1)\,q(x_1\mid y)
}{
p_{t_i}(x\mid y)
}\,
\dd x_1
\cdot
\grad \psi(x)\,
\hat p_{t_i}(x\mid y;\vec{\omega}_i)\,\dd x \right)^2,
\label{eq:endpoint-conditioned-objective}
\end{align}
an identity that is exact at the population level.
\end{theorem}

\begin{proof}[Proof sketch]
The posterior weights integrate to one in \(x_1\), so the guided field -
which carries no endpoint dependence - may be moved inside the endpoint integral. Writing the
marginal field \(u_{t_i}(\cdot\mid y)\) as its posterior average over \(x_1\) via
Equation~\eqref{eq:marginal-vector-t-cond} and merging the two integrals gives the displayed form.
The full proof is in Appendix~\ref{app:endpoint-conditioned-objective}.
\end{proof}

To estimate Equation~\eqref{eq:endpoint-conditioned-objective} we use a Monte Carlo approximation of these integrals: At time $t_i$, we have $N$ rollout particles $x_i^{(n)} \sim \hat p_{t_i}(\cdot \mid y; \vec{\omega}_i)$, and we sample $M$ endpoint samples $x_1^{(m)} \sim q(\cdot \mid y)$.

The practical procedure follows a recursive logic, progressing interval by interval through the solver grid. At each step $i$, we estimate the local coefficients of the quadratic objective $\tilde{\mathcal{L}}_i$ using the current generated rollout particles $\{x_i^{(n)}\}_{n=1}^N$. We then compute the optimal local guidance scale $\omega_i^\star$ via Equation~\eqref{eq:exact-local-selector}, commit this value for the duration of the interval $[t_i, t_{i+1})$, and advance the ODE solver one step to obtain the particles at $t_{i+1}$. This process continues until the terminal time $t_T = 1$ is reached. An implementation of this procedure is detailed in Appendix~\ref{app:algorithm}.

Note that the backbone \(\theta\) is never modified, but calibration is a \emph{model-owner} operation: it needs
endpoint samples \(x_1\sim q(\cdot\mid y)\) and a known path construction that makes
\(p_t(x\mid x_1,y)\) and \(u_t(x\mid x_1,y)\) evaluable (Appendix~\ref{app:conditional-paths}). Once fitted, none of these quantities is needed at deployment: the
stored schedule is a list of \(T\) scalars consumed exactly like a constant CFG scale.

Moreover, Proposition~\ref{prop:path-correctness} characterizes the \emph{ideal} target: weak equivalence for every \(\psi\in C_c^\infty(\R^d)\) and almost every \(s\). An implemented schedule realizes a finite counterpart of this target, and its relation to the real probability-path discrepancy is quantified empirically in Appendix~\ref{app:residual-diagnostic}.

\subsection{Usage modes}

This framework introduces an \textbf{on-line selector}: the guidance value applied to each interval is adaptively chosen based on the actual distribution $\hat{p}_{t_i}$ realized by the frozen model and solver up to that point. This approach leads to two primary deployment modes:

\begin{itemize}
    \item \textbf{Online Adaptive Selection:} In this mode, the guidance scale is recalibrated dynamically during every sampling run. This is particularly effective for high-fidelity generation where the specific drift of a sample batch must be corrected in real-time.
    \item \textbf{Offline Calibrated Scheduling:} For a fixed backbone model, solver, and grid, the selector can be run once on a representative set of initial samples $x_0 \sim p_0$. The resulting sequence of optimal scales $\vec{\omega}^\star$ is stored as a piecewise-constant schedule. This schedule can then be reused across all subsequent deployments, providing a \emph{pay-once, use-many-times} solution that enjoys the corrective benefits of our method without the overhead of computing the selector at every inference step.
\end{itemize}
By decoupling the estimation of the drift from the generative inference, we provide a flexible mechanism to stabilize probability paths in both compute-constrained and quality-critical settings.

\section{Experimental Validation}
\label{sec:experiments}

We evaluate the selector on a class-conditional Gaussian-mixture testbed. This setting is deliberately controlled: the endpoint law, intermediate conditional marginals, and posterior weights are available in closed form. It therefore lets us test the main claim of Section~\ref{sec:method}: the selected guidance schedule should improve the generated rollout
\(\hat p_{t_i}(\cdot\mid y;\vec\omega_i)\), not only the terminal samples at \(t=1\).

All schedules are fitted separately for each class, and evaluation metrics are aggregated over classes using the class frequencies. Unless stated otherwise, each method is evaluated over three inference seeds, with matched prior latents across methods within each seed. Schedules are fitted using three independent fitting seeds. Downstream results report the mean and sample standard deviation over inference seeds. Full experimental details, including flow variants, test-function families, stabilization, metrics, and implementation defaults, are provided in Appendix~\ref{app:exp-details}. Ablations are provided in Appendix~\ref{app:ablations}.

\paragraph{Online path alignment.}
We first test whether the selected guidance values improve the path followed by the sampler. Since the exact conditional marginal \(p_{t_i}(\cdot\mid y)\) is known at every solver time, we measure the discrepancy between \(p_{t_i}(\cdot\mid y)\) and the generated rollout \(\hat p_{t_i}(\cdot\mid y;\vec\omega_i)\) along the full trajectory. Figure~\ref{fig:online-path-alignment} compares three rollouts: one generated with the true conditional velocity \(u_t(\cdot\mid y)\), one generated with the learned conditional field \(v_t^\theta(\cdot\mid y)\), which corresponds to \(\omega\equiv1\), and one generated by the practical selector. Endpoint discrepancies at \(t \approx 1\) are reported in Table~\ref{tab:online-path-discrepancy-full}.

\begin{figure*}[t]
    \centering

    \begin{minipage}[t]{0.99\linewidth}
        \centering
        \includegraphics[width=\linewidth]{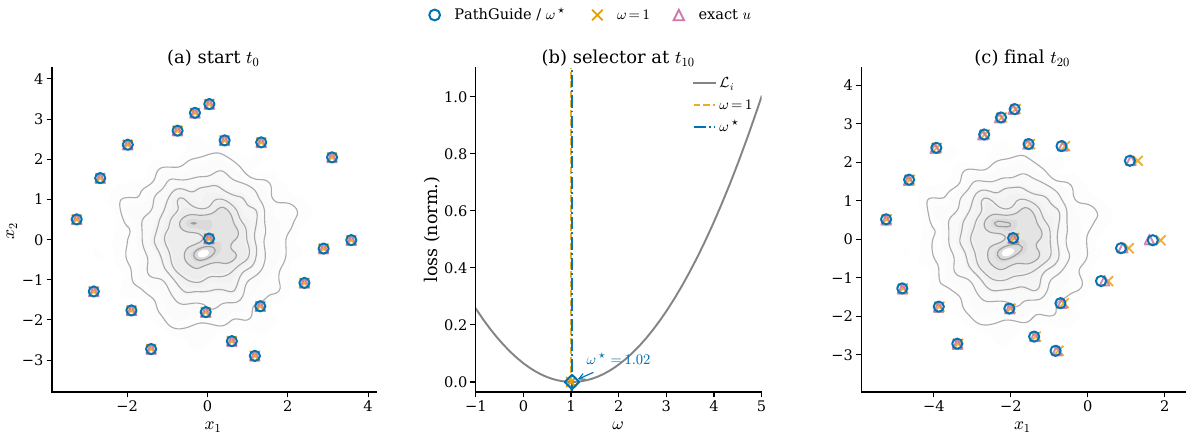}\\[-0.2em]
        (a)
    \end{minipage}

    \vspace{0.2em}

    \begin{minipage}[t]{0.99\linewidth}
        \centering
        \includegraphics[width=\linewidth]{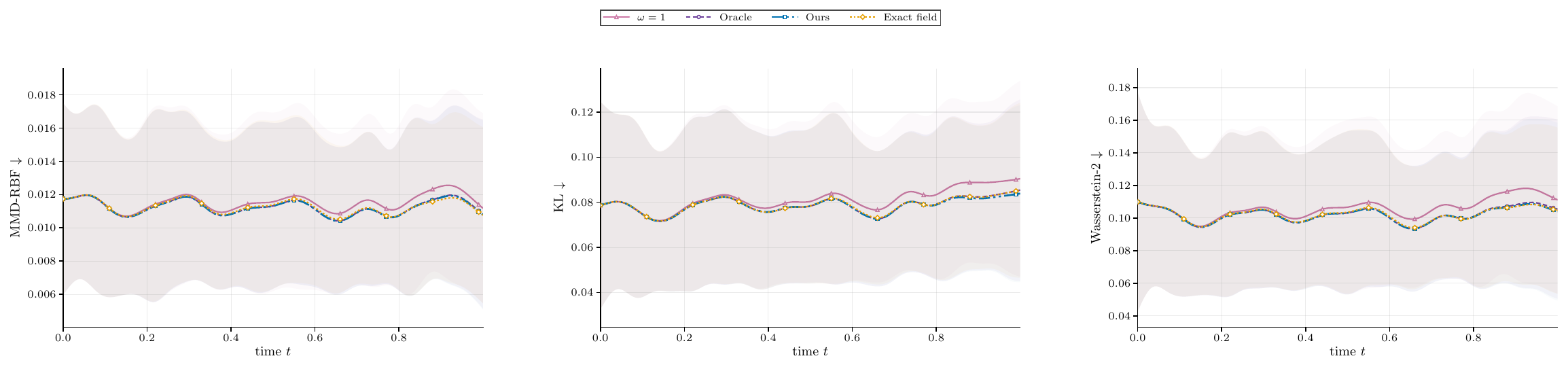}\\[-0.2em]
        (b)
    \end{minipage}
\caption{
    Comparison of rollout trajectories against the exact conditional marginal $p_{t_i}(\cdot\mid y)$. 
    (a) Online samples at the start ($t_0$) and near the end ($t_{T-1}$) of a $T=20$ step rollout, together with the quadratic local objective and the selected value at the midpoint time $t_{T/2}$.
    (b) Quantitative discrepancy measured by $\mathrm{MMD}$, $\mathrm{KL}$, and $\mathrm{W}_2$.
    The true-velocity rollout serves as a discretization reference, since finite-step integration can deviate from the analytic marginal path even when the exact conditional velocity is used. 
    While the plain conditional learned field ($\omega \equiv 1$) accumulates drift, our practical selector effectively calibrates the rollout to track the true path more closely.
}
    \label{fig:online-path-alignment}
\end{figure*}

\paragraph{Endpoint quality against guidance baselines.}
We next test whether improved path alignment translates into better terminal samples. Table~\ref{tab:baseline-comparison} compares the final generated distribution at \(t=1\) against plain CFG~\citep{ho2022classifierfree,zheng2023guidedflows}, CFG-Zero$^\ast$~\citep{fan2025cfgzerostar}, CFG-MP~\citep{cai2026cfgmp}, and Rectified-CFG++~\citep{saini2025rectifiedcfgflowbasedmodels}. Appendix~\ref{app:pointwise-control} adds a control that isolates the weak-form criterion itself, selecting the scale from the same endpoint-conditioned estimates by direct pointwise projection instead. The schedule for our method is fitted once on a reference solver grid and then reused on fresh inference seeds, matching the offline calibrated scheduling mode described in Section~\ref{sec:method}, isolating the benefit of selecting \(\vec\omega\). All methods in a comparison are run under identical conditions: the same frozen velocity field, endpoint law, solver grid, initial latents, and inference seeds, with each baseline tuned over its own hyperparameter sweep (Appendix~\ref{app:baseline-sweeps}).  We also evaluate the same protocol on MNIST handwritten digits \citep{mnist}. Table~\ref{tab:mnist-baseline-comparison} shows that the fitted schedules remain effective in pixel space for both Optimal Transport (OT) and Rectified Flow (RF) variants.

\begin{table*}[t]
    \caption{
    Generation quality against guidance baselines for all flow variants. Each metric is reported for \(T=200\) and \(T=500\); lower is better. All methods are evaluated using \(2^{14}\) generated samples per seed. Our method is fitted once using \(M=N=2^{14}\), independently of the evaluation samples. Each row reports the best configuration selected from that method's hyperparameter sweep, using the selection rule described in Appendix~\ref{app:baseline-sweeps}. Full sweeps including standard deviations are in Appendix~\ref{app:baseline-sweeps-tables}.
    }
    \label{tab:baseline-comparison}
    \vspace{0.35em}
    \centering
    \small
    \setlength{\tabcolsep}{2pt}
    \renewcommand{\arraystretch}{1.15}
    \begin{tabular}{@{}llcccccc@{}}
        \toprule
    & & \multicolumn{2}{c}{KL\(\downarrow\)}
      & \multicolumn{2}{c}{\(\mathrm{W}_2\downarrow\)}
      & \multicolumn{2}{c}{MMD\(\downarrow\)} \\
    \cmidrule(lr){3-4}
    \cmidrule(lr){5-6}
    \cmidrule(lr){7-8}
    Variant & Method
      & \(T=200\) & \(T=500\)
      & \(T=200\) & \(T=500\)
      & \(T=200\) & \(T=500\) \\
    \midrule
    RF \citep{liu2022flowstraightfastlearning}
    & Plain CFG       & \(0.0044\) & \(0.0044\) & \(0.0060\) & \(0.0060\) & \(0.0006\) & \(0.0006\) \\
    & CFG-Zero\(^*\) & \(0.0046\) & \(0.0045\) & \(0.0064\) & \(0.0062\) & \(0.0006\) & \(0.0006\) \\
    & CFG-MP         & \(6.2232\) & \(1.8740\) & \(9.5319\) & \(2.9592\) & \(0.6234\) & \(0.2987\) \\
    & R-CFG++        & \(0.0191\) & \(0.0185\) & \(0.0248\) & \(0.0243\) & \(0.0020\) & \(0.0019\) \\
    & PathGuide (Ours)           & \boldmath{\(0.0039\)} & \boldmath{\(0.0038\)} & \boldmath{\(0.0056\)} & \boldmath{\(0.0055\)} & \boldmath{\(0.0005\)} & \boldmath{\(0.0005\)} \\

    \midrule

    I-CFM \citep{tong2024minibatchot}
    & Plain CFG       & \(0.0084\) & \(0.0083\) & \(0.0112\) & \(0.0111\) & \(0.0012\) & \(0.0012\) \\
    & CFG-Zero\(^*\) & \(0.0084\) & \(0.0083\) & \(0.0114\) & \(0.0111\) & \(0.0012\) & \(0.0012\) \\
    & CFG-MP         & \(0.1197\) & \(0.1191\) & \(0.1853\) & \(0.1859\) & \(0.0168\) & \(0.0167\) \\
    & R-CFG++        & \(0.0192\) & \(0.0185\) & \(0.0252\) & \(0.0246\) & \(0.0024\) & \(0.0023\) \\
    & PathGuide (Ours)           & \boldmath{\(0.0076\)} & \boldmath{\(0.0073\)} & \boldmath{\(0.0099\)} & \boldmath{\(0.0096\)} & \boldmath{\(0.0012\)} & \boldmath{\(0.0012\)} \\

    \midrule

    OT \citep{lipman2023flowmatching}
    & Plain CFG       & \(0.0082\) & \(0.0080\) & \(0.0111\) & \(0.0110\) & \(0.0012\) & \(0.0012\) \\
    & CFG-Zero\(^*\) & \(0.0081\) & \(0.0080\) & \(0.0112\) & \(0.0110\) & \boldmath{\(0.0012\)} & \(0.0011\) \\
    & CFG-MP         & \(0.1215\) & \(0.1208\) & \(0.1889\) & \(0.1895\) & \(0.0171\) & \(0.0170\) \\
    & R-CFG++        & \(0.0199\) & \(0.0192\) & \(0.0266\) & \(0.0260\) & \(0.0024\) & \(0.0024\) \\
    & PathGuide (Ours)           & \boldmath{\(0.0076\)} & \boldmath{\(0.0073\)} & \boldmath{\(0.0102\)} & \boldmath{\(0.0099\)} & \(0.0012\) & \boldmath{\(0.0011\)} \\

    \midrule

    VP \citep{song2021scorebasedsde}
    & Plain CFG       & \(0.0030\) & \(0.0030\) & \(0.0052\) & \(0.0052\) & \(0.0005\) & \(0.0005\) \\
    & CFG-Zero\(^*\) & \(0.0013\) & \(0.0013\) & \(0.0019\) & \(0.0019\) & \(0.0003\) & \(0.0003\) \\
    & CFG-MP         & \(2.9991\) & \(1.1612\) & \(4.7929\) & \(1.6279\) & \(0.4827\) & \(0.2012\) \\
    & R-CFG++        & \(0.0290\) & \(0.0289\) & \(0.0388\) & \(0.0387\) & \(0.0029\) & \(0.0028\) \\
    & PathGuide (Ours)           & \boldmath{\(0.0011\)} & \boldmath{\(0.0011\)} & \boldmath{\(0.0015\)} & \boldmath{\(0.0015\)} & \boldmath{\(0.0003\)} & \boldmath{\(0.0003\)} \\
    \bottomrule
  \end{tabular}
    \vspace{1.0em}
\end{table*}

\begin{table*}[t]
    \caption{
    MNIST image-generation quality for RF and OT flow variants.
    We report FID computed with \texttt{torch-fidelity==0.4.0}.
    Lower is better. All methods are evaluated over three evaluation seeds using \(2^{13}\)
    generated samples per seed, with \(M=2^{11}\), \(N=2^{12}\), and \(T=50\).
    Each row reports the configuration with the lowest mean FID within that method's sweep.
    Full sweeps are reported in Table~\ref{tab:mnist-baseline-sweeps-t50}.
    On matched inference latents, the paired PathGuide-minus-tuned-CFG FID difference is
    \(-0.103\) (95\% CI \([-0.186,-0.021]\)) for RF and \(-0.175\) (\([-0.250,-0.101]\)) for OT.
    }
    \label{tab:mnist-baseline-comparison}
    \vspace{0.35em}
    \centering
    \small
    \setlength{\tabcolsep}{8pt}
    \renewcommand{\arraystretch}{1.08}
    \begin{tabular}{@{}lclc@{}}
        \toprule
        \multicolumn{2}{c}{RF \citep{liu2022flowstraightfastlearning}} 
        & \multicolumn{2}{c}{OT \citep{lipman2023flowmatching}} \\
        \cmidrule(r){1-2}
        \cmidrule(l){3-4}
        Method & FID\(\downarrow\)
        & Method & FID\(\downarrow\) \\
        \midrule
        Plain CFG & \(15.802\)
        & Plain CFG & \(7.578\) \\
        CFG-Zero\(^\ast\) & \(26.534\)
        & CFG-Zero\(^\ast\) & \(11.327\) \\
        CFG-MP & \(30.951\)
        & CFG-MP & \(24.475\) \\
        R-CFG++ & \(16.979\)
        & R-CFG++ & \(15.038\) \\
        PathGuide (Ours) & \(\mathbf{15.699}\)
        & PathGuide (Ours) & \(\mathbf{7.375}\) \\
        \bottomrule
    \end{tabular}
\end{table*}
\paragraph{Reuse under coarser inference.}
The offline schedule is more useful if it can be reused under cheaper inference. We therefore fit schedules on a reference grid \(T\), compress them to coarser grids \(T^{-}\leq T\) aggregating with an average over \(\omega_i\) included in each corresponding coarse interval, and evaluate the resulting terminal distributions using matched latents. Figure~\ref{fig:gm_infer_at_tminus} reports the degradation as the inference grid is coarsened. The fitted schedules degrade gracefully \((\leq4\%\)), supporting the "fit once, reuse later" interpretation of the method.

\begin{figure*}[t!]
    \centering
    \includegraphics[width=0.99\linewidth]{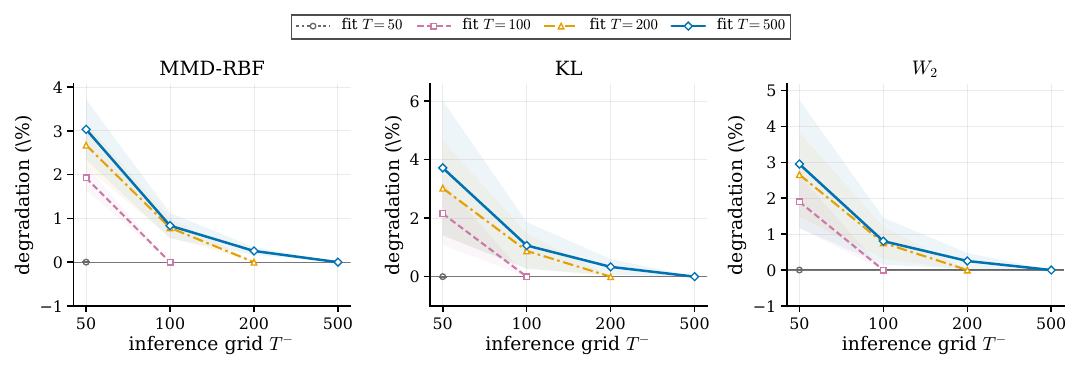}
    \caption{
    Reusing a schedule fitted at grid \(T\) on coarser inference grids \(T^{-}\leq T\). Each curve fixes the fitting grid and varies the deployed inference grid. \(
\mathrm{degradation}
=
100\cdot(
\mathrm{coarse\ result}
-
\mathrm{full\ result})/
\mathrm{full\ result}
,
\) each result is averaged over the fitting seeds. Shaded regions show one standard deviation over the inference seeds. Lower is better for all metrics.
    }
    \label{fig:gm_infer_at_tminus}
\end{figure*}

\paragraph{Schedule diagnostics.}
We then inspect whether the fitted schedules follow the local weak-form objective derived in Section~\ref{sec:method}. Figure~\ref{fig:experiment4-vp-diffusion-class-0} shows the \(T=500\) analytic local-objective landscape over \((t_i,\omega)\) for a representative VP run, with fitted schedules overlaid for different solver grids. The fitted schedules follow the low-objective region and preserve the same global shape across fitting resolutions. The \emph{oracle schedule} uses the same interval-wise selector, but replaces the estimated posterior weights with the analytic posterior weights available in this controlled setting. Finer grids reduce estimated schedule variance and better track the oracle schedule.

\begin{figure*}[t]
    \centering
    \includegraphics[width=0.80\linewidth]{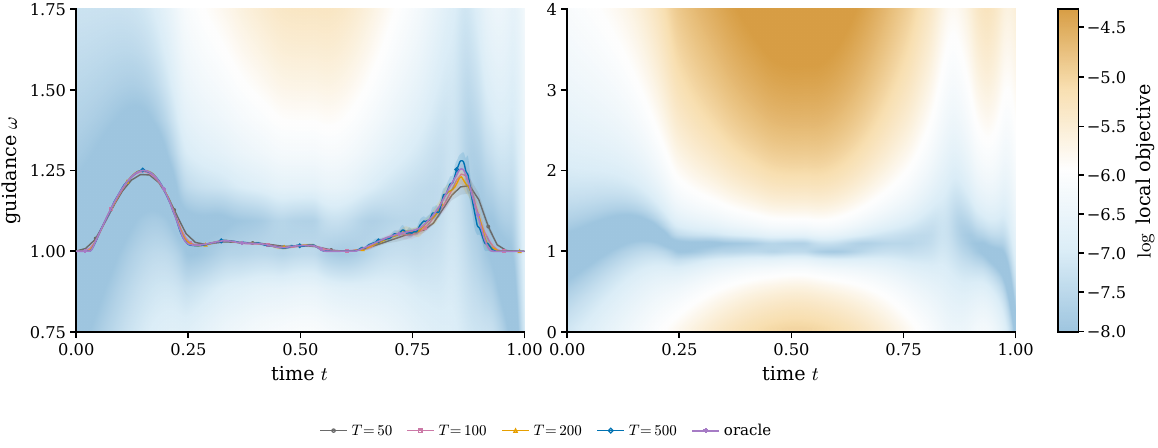}
    \caption{
    Schedule diagnostic for the VP flow, class \(0\). The background shows the analytic local-objective landscape over \((t_i,\omega)\). Curves show fitted schedules for different solver grids together with the oracle schedule. The practical schedules preserve the same global shape across resolutions and better follow low-objective regions as the grid is refined.}
    \label{fig:experiment4-vp-diffusion-class-0}
\end{figure*}

\paragraph{Monte Carlo scaling.}

Finally, we vary the estimator budget used by the practical selector. Figure~\ref{fig:gm_mn_diagonal_scaling} shows the diagonal sweep \(M=N\), keeping the trained model and solver grid fixed. Downstream metrics and schedule distance stabilize at moderate budgets, whereas runtime continues to increase. This indicates that most gains are obtained before the estimator becomes expensive. The full \((M,N)\) sweep in Figure~\ref{fig:app-gm-mn-heatmaps} further shows that endpoint samples are especially important, as they affect both the posterior weights and the endpoint-conditioned velocity estimates.

\begin{figure*}[h!]
    \centering
    \includegraphics[width=\linewidth]{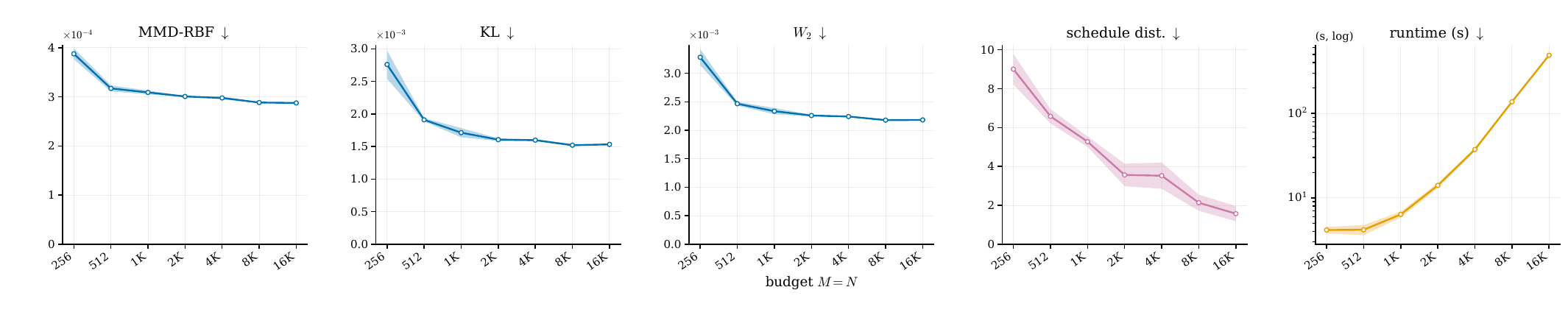}
    \caption{
    Monte Carlo scaling on the diagonal \(M=N\). Shaded regions show one standard deviation over fitting seeds. Downstream metrics and schedule distance stabilize at moderate budgets, while runtime continues to grow.
    }
    \label{fig:gm_mn_diagonal_scaling}
\end{figure*}

\section{Conclusion and Future Work} 
\label{sec:conclusions}
We introduced \emph{PathGuide}, a framework that reformulates scalar CFG selection as an on-policy transport problem grounded in the weak form of the continuity equation. By deriving a strictly quadratic local objective, we provide a closed-form selector that ensures generative rollouts remain faithful to the target conditional law. Despite its theoretical and empirical advantages, our method has several limitations. Once a schedule is fitted and amortized, inference has the same backbone-evaluation cost as standard CFG, however, online selection adds estimator overhead during sampling. Furthermore, the selector is local in time, and future work could extend it to a global optimal-control formulation. Moreover, the framework's sensitivity to the test-function family \(\mathcal{H}\) introduces a trade-off between discrepancy detection and computational variance (Appendix~\ref{app:test-functions}), suggesting the use of adaptive or learned kernels as an extension. While we focused on scalar CFG, the transport-based objective is naturally extensible to more expressive, non-linear control signals. Finally, our empirical scope is a two-class Gaussian mixture and MNIST. We therefore make no claim of generalization to large backbones, higher dimensions, or open-ended conditional generation; evaluating the weak-form objective on large-scale text-to-image and video benchmarks - particularly for reward-tilted or tempered distributions - remains a critical next step.

\bibliographystyle{plainnat}
\bibliography{refs}

\appendix

\newpage

\section{Additional background: Unifying Gaussian Flow-Matching and Diffusion Probability-Flows}
\label{app:diffusion-paths}

This section establishes the formal equivalence between the velocity fields used in flow-matching and the score-based representations in diffusion ODEs referred to in
Section~\ref{sec:background}.

Steering a generative process toward a condition \(y\) can be understood through the lens of a time-dependent classifier \(p_t(y\mid x)\) \citep{dhariwal2021classifierguidance}. To avoid training such a classifier explicitly, classifier-free guidance \citep{ho2021classifierfreeworkshop} uses the Bayes' rule decomposition
\begin{equation}
\nabla_x \log p_t(x \mid y)
=
\nabla_x \log p_t(x)
+
\nabla_x \log p_t(y \mid x),
\end{equation}
so the classifier gradient is approximated by the difference between the conditional and unconditional score fields. The guided score is then used inside the diffusion sampling dynamics. Equivalently, the same marginal probability path is generated by a deterministic probability-flow ODE whose vector field depends on the score, so applying CFG to the score induces a corresponding guided probability-flow vector field. This is the bridge from score-based diffusion to flow matching \citep{maoutsa2020interacting,lipman2023flowmatching, song2021scorebasedsde}.

Consider an endpoint-conditioned Gaussian path
\begin{equation}
p_t(x\mid x_1,y)
=
\mathcal N\!\left(x\mid \alpha_t x_1,\sigma_t^2 \I\right),
\qquad
t\in[0,1],
\end{equation}
with differentiable scalar schedules \(\alpha_t\) and \(\sigma_t>0\). Its
endpoint-conditioned score is
\begin{equation}
\grad_x \log p_t(x\mid x_1,y)
=
-\frac{x-\alpha_t x_1}{\sigma_t^2}.
\end{equation}
The corresponding endpoint-conditioned velocity field is
\begin{equation}
u_t(x\mid x_1,y)
=
\dot\alpha_t x_1
+
\frac{\dot\sigma_t}{\sigma_t}
\bigl(x-\alpha_t x_1\bigr).
\end{equation}
Equivalently, when \(\alpha_t\neq 0\),
\begin{equation}
u_t(x\mid x_1,y)
=
\frac{\dot\alpha_t}{\alpha_t}x
-
\sigma_t^2
\left(
\frac{\dot\sigma_t}{\sigma_t}
-
\frac{\dot\alpha_t}{\alpha_t}
\right)
\grad_x\log p_t(x\mid x_1,y).
\end{equation}
Thus Gaussian endpoint-conditioned flow-matching paths admit a
drift-plus-score representation.

Crucially, because the relationship between the velocity field and the score is affine in $x$ and the score term, this identity extends directly to the marginal laws via the linearity of expectation \citep{albergo2025stochasticinterpolants, zheng2023guidedflows}. Specifically, by integrating the conditioned field over the posterior $p(x_1 \mid x, y)$, we obtain the marginal conditional field:
\begin{equation}
    u_t(x \mid y) = \frac{\dot\alpha_t}{\alpha_t}x - \sigma_t^2 \left( \frac{\dot\sigma_t}{\sigma_t} - \frac{\dot\alpha_t}{\alpha_t} \right) \nabla_x \log p_t(x \mid y).
\end{equation}
The same identity holds for the unconditional field $u_t(x)$ by marginalizing over $y$. 

This resembles the deterministic probability-flow ODE notation used for
score-based diffusion, where a marginal path is generated by a field of the form
\begin{equation}
\frac{\dd x_t}{\dd t}
=
f(x_t,t)
-
\frac12 g(t)^2
\grad_x\log p_t(x_t\mid y).
\end{equation}

\section{Proofs}
\label{app:proofs}

This appendix collects the proof details used in Section~\ref{sec:method}.
All statements are written for a fixed condition \(y\). We distinguish the exact
continuous-time objects from the empirical estimators used by the practical
algorithm.

\subsection{Standing assumptions}
\label{app:assumptions}

We use the following assumptions throughout the proofs.

\begin{assumption}[Regularity of conditional paths]
\label{ass:conditional-path-regularity}
For \(t\in[0,1]\), the endpoint-conditioned path
\(p_t(\cdot\mid x_1,y)\) and its vector field
\(u_t(\cdot\mid x_1,y)\) satisfy the continuity equation
\begin{equation}
    \partial_t p_t(x\mid x_1,y)
    +
    \divg\!\bigl(
        p_t(x\mid x_1,y)u_t(x\mid x_1,y)
    \bigr)
    =
    0
\end{equation}
in the weak sense. Moreover, the integrability and regularity conditions needed
to exchange differentiation, divergence, and integration over \(x_1\) hold.
\end{assumption}

\begin{assumption}[Endpoint posterior weights]
\label{ass:endpoint-posterior}
The marginal conditional density \eqref{eq:marginal-probability-density-t-cond}
is finite and positive on the region where the generated rollout law is
evaluated. We define
\begin{equation}
    \pi_t(x_1\mid x,y)
    :=
    \frac{
        p_t(x\mid x_1,y)q(x_1\mid y)
    }{
        p_t(x\mid y)
    }.
\end{equation}
Outside the support where \(p_t(x\mid y)>0\), the value of
\(\pi_t(\cdot\mid x,y)\) is immaterial for the objectives below.
\end{assumption}

\begin{assumption}[Generated rollout law]
\label{ass:rollout-law}
For every admissible guidance schedule, the generated rollout law
\(\hat p_t(\cdot\mid y;\omega_{[0,t)})\) is weakly continuous in \(t\) and
satisfies the weak continuity equation driven by the guided field
\(v_t^\theta(\cdot\mid y;\omega_t)\).
The law at time \(t\) depends only on the already committed past schedule
\(\omega_{[0,t)}\).
\end{assumption}

\begin{assumption}[Uniqueness]
\label{ass:ce-uniqueness}
The exact conditional path $p_t(\cdot\mid y)$ is the unique weak solution of
the continuity equation driven by the exact marginal conditional field
$u_t(\cdot\mid y)$, with initial law $p_0$. Specifically, following the 
theory of \citet{DiPerna1989} and \citet{Ambrosio2004}, we assume:
\begin{enumerate}
    \item $u_t(\cdot\mid y) \in L^1([0,1]; BV_{loc}(\mathbb{R}^d))$ with at most linear growth in $x$.
    \item $\divg u_t(\cdot\mid y) \in L^1([0,1]; L^\infty(\mathbb{R}^d))$.
    \item $p_t(\cdot\mid y) \in L^\infty([0,1] \times \mathbb{R}^d)$.
\end{enumerate}
\end{assumption}

\subsection{Endpoint-conditioned paths imply the marginal continuity equation}
\label{app:marginal-ce}

\begin{proposition}[Marginal conditional continuity equation]
\label{prop:app-marginal-ce}
Under Assumptions~\ref{ass:conditional-path-regularity}
and~\ref{ass:endpoint-posterior},
\begin{equation}
    \partial_t p_t(x\mid y)
    +
    \divg\!\bigl(
        p_t(x\mid y)u_t(x\mid y)
    \bigr)
    =
    0 .
\end{equation}
\end{proposition}

\begin{proof}
By differentiation under the endpoint integral,
\begin{equation}
    \partial_t p_t(x\mid y)
    =
    \int_{\R^d}
    \partial_t p_t(x\mid x_1,y)q(x_1\mid y)\,\dd x_1 .
\end{equation}
Using the endpoint-conditioned continuity equation,
\begin{equation}
    \partial_t p_t(x\mid y)
    =
    -
    \int_{\R^d}
    \divg\!\bigl(
        p_t(x\mid x_1,y)u_t(x\mid x_1,y)
    \bigr)
    q(x_1\mid y)\,\dd x_1 .
\end{equation}
The divergence acts on \(x\), so by the regularity assumption it may be moved
outside the endpoint integral:
\begin{equation}
    \partial_t p_t(x\mid y)
    =
    -
    \divg\!\left(
    \int_{\R^d}
        p_t(x\mid x_1,y)u_t(x\mid x_1,y)q(x_1\mid y)\,\dd x_1
    \right).
\end{equation}
Using the definition of \(\pi_t(x_1\mid x,y)\),
\begin{equation}
    \int_{\R^d}
        p_t(x\mid x_1,y)u_t(x\mid x_1,y)q(x_1\mid y)\,\dd x_1
    =
    p_t(x\mid y)u_t(x\mid y).
\end{equation}
This recovers the marginal continuity equation described in the Flow Matching paragraph of Section~\ref{sec:background}, a result previously established by \citet{lipman2023flowmatching}.
\end{proof}

\subsection{Weak form used by the local objective}
\label{app:weak-form}

\begin{proposition}[Weak continuity identity]
\label{prop:app-weak-form}
Let \(\rho_t\) and \(b_t\) satisfy
\begin{equation}
    \partial_t\rho_t+\divg(\rho_t b_t)=0
\end{equation}
in the distributional sense on \([0,1]\times\R^d\). Then, for every
\(\psi\in C_c^1(\R^d)\), the map
\begin{equation}
    t\mapsto \int_{\R^d}\psi(x)\rho_t(x)\,\dd x
\end{equation}
is absolutely continuous and, for almost every \(t\),
\begin{equation}
    \frac{\dd}{\dd t}
    \int_{\R^d}\psi(x)\rho_t(x)\,\dd x
    =
    \int_{\R^d}
    \grad\psi(x)\cdot b_t(x)\rho_t(x)\,\dd x .
\end{equation}
\end{proposition}

\begin{proof}
Test the distributional continuity equation with
\(\varphi(t,x)=\xi(t)\psi(x)\), where
\(\xi\in C_c^1((0,1))\) and \(\psi\in C_c^1(\R^d)\). This yields
\begin{equation}
    \int_0^1 \xi'(t)
    \left(
        \int_{\R^d}\psi(x)\rho_t(x)\,\dd x
    \right)\dd t
    +
    \int_0^1 \xi(t)
    \left(
        \int_{\R^d}
        \grad\psi(x)\cdot b_t(x)\rho_t(x)\,\dd x
    \right)\dd t
    =
    0 .
\end{equation}
Hence the scalar map
\(t\mapsto \int\psi\,\rho_t\) has weak derivative
\(t\mapsto \int \grad\psi\cdot b_t\,\rho_t\), giving the desired
absolutely continuous representative and pointwise identity for almost every
\(t\). For a complete treatment, see \citet{ambrosio2005gradient, santambrogio2015optimaltf}.
\end{proof}

\subsection{Ideal weak equivalence implies path correctness}
\label{app:path-correctness}

\begin{proof}[Proof of Proposition~\ref{prop:path-correctness}]
Fix \(\psi\in C_c^\infty(\R^d)\). By Assumption~\ref{ass:rollout-law}, the
generated law satisfies
\begin{equation}
    \frac{\dd}{\dd s}
    \int_{\R^d}\psi(x)\hat p_s(x)\,\dd x
    =
    \int_{\R^d}
    \grad\psi(x)\cdot
    v_s^\theta(x\mid y;\omega_s)
    \hat p_s(x)\,\dd x
\end{equation}
for almost every \(s\in[0,t]\). The assumed weak equivalence in
Proposition~\ref{prop:path-correctness} replaces the right-hand side by
\begin{equation}
    \int_{\R^d}
    \grad\psi(x)\cdot u_s(x\mid y)\hat p_s(x)\,\dd x .
\end{equation}
Therefore \(\hat p_s\) is a weak solution of the continuity equation driven by
\(u_s(\cdot\mid y)\). Since \(\hat p_0=p_0\), Assumption~\ref{ass:ce-uniqueness}
implies
\begin{equation}
    \hat p_s(\cdot\mid y;\omega_{[0,s)})
    =
    p_s(\cdot\mid y),
    \qquad s\in[0,t].
\end{equation}
\end{proof}

\subsection{Exact local weak-form objective}
\label{app:local-objective}

The body defines the local objective at solver time \(t_i\), weighted by the
current generated rollout law
\(\hat p_{t_i}(\cdot\mid y;\vec\omega_i)\), where
\(\vec\omega_i=(\omega_0,\ldots,\omega_{i-1})\) is fixed before choosing
\(\omega_i\). For a test function \(\psi\), define the signed weak-form
mismatch
\begin{equation}
    G_i(\omega_i\mid y;\vec\omega_i,\psi)
    :=
    \int_{\R^d}
    \Bigl(
        u_{t_i}(x\mid y)
        -
        v_{t_i}^\theta(x\mid y;\omega_i)
    \Bigr)\cdot
    \grad\psi(x)\,
    \hat p_{t_i}(x\mid y;\vec\omega_i)\,\dd x .
\end{equation}
The exact local objective is
\begin{equation}
    \mathcal L_i(\omega_i\mid y;\vec\omega_i,\psi)
    :=
    \bigl|
        G_i(\omega_i\mid y;\vec\omega_i,\psi)
    \bigr|^2 .
    \label{eq:app-local-objective}
\end{equation}
This objective is on-policy: its weighting distribution is the rollout produced
by the already committed past guidance values, not the ideal path.

\subsection{Endpoint-conditioned representation of the exact objective}
\label{app:endpoint-conditioned-objective}

\begin{proposition}[Endpoint-conditioned objective]
\label{prop:app-endpoint-conditioned-objective}
Under Assumptions~\ref{ass:conditional-path-regularity}
and~\ref{ass:endpoint-posterior}, the signed quantity in
\eqref{eq:app-local-objective} admits the endpoint-conditioned form
\begin{equation}
\begin{aligned}
    G_i(\omega_i\mid y;\vec\omega_i,\psi)
    =
    \int_{\R^d}
    \int_{\R^d}
    &
    \Bigl(
        u_{t_i}(x\mid x_1,y)
        -
        v_{t_i}^\theta(x\mid y;\omega_i)
    \Bigr)
    \cdot \grad\psi(x)
    \\
    &\times
    \pi_{t_i}(x_1\mid x,y)\,
    \hat p_{t_i}(x\mid y;\vec\omega_i)
    \,\dd x_1\,\dd x .
\end{aligned}
\end{equation}
Consequently, \(\mathcal L_i\) can be evaluated from endpoint-conditioned
velocities and posterior endpoint weights, which proves
Theorem~\ref{thm:endpoint-conditioned-objective}.
\end{proposition}

\begin{proof}
We proceed in four steps.

\emph{(i) Posterior normalization.} By Assumption~\ref{ass:endpoint-posterior},
\(\pi_{t_i}(x_1\mid x,y)=p_{t_i}(x\mid x_1,y)q(x_1\mid y)/p_{t_i}(x\mid y)\)
is a probability density in \(x_1\) wherever \(p_{t_i}(x\mid y)>0\), so
\(\int_{\R^d}\pi_{t_i}(x_1\mid x,y)\,\dd x_1=1\).

\emph{(ii) Marginal field as a posterior average.} Eq.~\eqref{eq:marginal-vector-t-cond} reads
\(u_{t_i}(x\mid y)=\int_{\R^d}u_{t_i}(x\mid x_1,y)\,\pi_{t_i}(x_1\mid x,y)\,\dd x_1\).

\emph{(iii) The guided field carries no endpoint dependence.}
\(v_{t_i}^\theta(x\mid y;\omega_i)\) does not depend on \(x_1\), so by (i) it may be written as
\(\int_{\R^d}v_{t_i}^\theta(x\mid y;\omega_i)\,\pi_{t_i}(x_1\mid x,y)\,\dd x_1\).

\emph{(iv) Combine.} Substituting (ii) and (iii) into the definition of \(G_i\) and merging the two
\(x_1\)-integrals - permitted by Assumption~\ref{ass:conditional-path-regularity} - gives the displayed
formula. Squaring the signed quantity gives the endpoint-conditioned representation of \(\mathcal L_i\).

The identity is exact at the population level. In the estimator of Appendix~\ref{app:algorithm}, both
\(\pi_{t_i}\) and the outer integral are replaced by Monte Carlo averages.
\end{proof}

\subsection{Affine form under scalar CFG}
\label{app:affine-objective}

\begin{proposition}[Quadratic local objective under scalar CFG]
\label{prop:app-quadratic-objective}
Assume scalar CFG:
\begin{equation}
    v_t^\theta(x\mid y;\omega)
    =
    v_t^\theta(x\mid \varnothing)
    +
    \omega
    \Bigl(
        v_t^\theta(x\mid y)-v_t^\theta(x\mid \varnothing)
    \Bigr).
\end{equation}
For fixed \(y,\vec\omega_i,\psi\), define
\begin{equation}
    A_i(\psi)
    :=
    \int_{\R^d}
    \Bigl(
        u_{t_i}(x\mid y)
        -
        v_{t_i}^\theta(x\mid \varnothing)
    \Bigr)\cdot
    \grad\psi(x)\,
    \hat p_{t_i}(x\mid y;\vec\omega_i)\,\dd x
\end{equation}
and
\begin{equation}
    B_i(\psi)
    :=
    \int_{\R^d}
    \Bigl(
        v_{t_i}^\theta(x\mid y)
        -
        v_{t_i}^\theta(x\mid \varnothing)
    \Bigr)\cdot
    \grad\psi(x)\,
    \hat p_{t_i}(x\mid y;\vec\omega_i)\,\dd x .
\end{equation}
Then
\begin{equation}
    \mathcal L_i(\omega_i\mid y;\vec\omega_i,\psi)
    =
    \bigl|A_i(\psi)-\omega_i B_i(\psi)\bigr|^2 .
\end{equation}
In particular, the exact local objective is quadratic in \(\omega_i\).
\end{proposition}

\begin{proof}
Insert the scalar CFG field into \(G_i\):
\begin{equation}
\begin{aligned}
    G_i(\omega_i\mid y;\vec\omega_i,\psi)
    =
    \int_{\R^d}
    &\Bigl(
        u_{t_i}(x\mid y)
        -
        v_{t_i}^\theta(x\mid \varnothing)
    \\
    &\quad
        -
        \omega_i
        \bigl(
            v_{t_i}^\theta(x\mid y)
            -
            v_{t_i}^\theta(x\mid \varnothing)
        \bigr)
    \Bigr)\cdot
    \grad\psi(x)\,
    \hat p_{t_i}(x\mid y;\vec\omega_i)\,\dd x .
\end{aligned}
\end{equation}
Collecting the terms independent of \(\omega_i\) and those multiplying
\(\omega_i\) gives \(G_i=A_i-\omega_i B_i\). Squaring gives the claim.
\end{proof}

\subsection{Exact local selector}
\label{app:exact-local-selector-proof}

\begin{proof}[Proof of Corollary~\ref{thm:exact-local-selector}]
By Proposition~\ref{prop:app-quadratic-objective},
\begin{equation}
    \mathcal L_i(\omega_i)
    =
    \bigl|A_i-\omega_i B_i\bigr|^2 .
\end{equation}
For a single test function, \(A_i,B_i\in\R\), so
\begin{equation}
    \mathcal L_i(\omega_i)
    =
    A_i^2-2\omega_iA_iB_i+\omega_i^2B_i^2 .
\end{equation}
If \(B_i\neq0\), differentiating and setting the derivative to zero gives
\begin{equation}
    \omega_i^\star
    =
    \frac{A_iB_i}{B_i^2}
    =
    \frac{A_i}{B_i}.
\end{equation}
Substituting the definitions of \(A_i\) and \(B_i\) gives
Eq.~\eqref{eq:exact-local-selector} in the body.

If \(B_i=0\), then the objective is independent of \(\omega_i\). Hence every
admissible guidance value is optimal. In practical use, this degenerate case is
handled by a denominator floor as specified in \ref{app:stabilization}.
\end{proof}

\subsection{From the local objective to the interval-wise implementation}
\label{app:interval-wise-selector-proof}

The selector in Corollary~\ref{thm:exact-local-selector} is written at a fixed
time \(t_i\). In numerical sampling, the guidance value is held fixed over the
solver interval \([t_i,t_{i+1})\). We therefore use the left-endpoint
discretization of the local weak-form objective on each interval.

\section{Practical algorithm}
\label{app:algorithm}

The algorithm can be used in two modes:
\emph{online selection}, where the guidance value is selected during the same
sampling run, and \emph{offline schedule fitting}, where a schedule is fitted
once and then reused for later inference with the same backbone, solver, grid,
and condition.

Here \(T\) is the number of solver intervals, \(M\) is the number of endpoint
samples, \(N\) is the number of rollout particles, and \(L\) is the number of
weak test functions. 

\paragraph{Posterior weights.}
At interval \(i\), suppose the current rollout particles are
\begin{equation}
x_i^{(n)}
\sim
\hat p_{t_i}(\cdot\mid y;\vec\omega_i),
\qquad
n=1,\ldots,N,
\end{equation}
where \(\vec\omega_i=(\omega_0,\ldots,\omega_{i-1})\) is the already committed
schedule.  Draw endpoint samples
\begin{equation}
x_{1,i}^{(m)}\sim q(\cdot\mid y),
\qquad
m=1,\ldots,M.
\end{equation}
For each rollout particle \(x_i^{(n)}\), define the empirical posterior weights
\begin{equation}
\label{eq:app-practical-posterior-weights}
\pi_i^{(n,m)}
:=
\frac{
    p_{t_i}\!\left(x_i^{(n)}\mid x_{1,i}^{(m)},y\right)
}{
    \sum_{r=1}^{M}
    p_{t_i}\!\left(x_i^{(n)}\mid x_{1,i}^{(r)},y\right)
}.
\end{equation}

\paragraph{Multi-test coefficient estimator.}
Let
\begin{equation}
\label{eq:multi-test-family}
\Psi=\{\psi_\ell\}_{\ell=1}^{L}
\end{equation}
be the chosen test family, and let \(\Delta t_i=t_{i+1}-t_i\).  For each
\(\psi_\ell\), define
\begin{equation}
\label{eq:app-practical-A}
\widetilde A_i^{(\ell)}
:=
\Delta t_i
\frac1N
\sum_{n=1}^{N}
\left[
    \sum_{m=1}^{M}
    \pi_i^{(n,m)}
    \Bigl(
        u_{t_i}(x_i^{(n)}\mid x_{1,i}^{(m)},y)
        -
        v_{t_i}^{\theta}(x_i^{(n)}\mid \varnothing)
    \Bigr)
\right]
\cdot
\grad\psi_\ell(x_i^{(n)}),
\end{equation}
and
\begin{equation}
\label{eq:app-practical-B}
\widetilde B_i^{(\ell)}
:=
\Delta t_i
\frac1N
\sum_{n=1}^{N}
\Bigl(
    v_{t_i}^{\theta}(x_i^{(n)}\mid y)
    -
    v_{t_i}^{\theta}(x_i^{(n)}\mid \varnothing)
\Bigr)
\cdot
\grad\psi_\ell(x_i^{(n)}).
\end{equation}
Stack these coefficients as
\begin{equation}
\widetilde A_i
=
\bigl(
\widetilde A_i^{(1)},\ldots,\widetilde A_i^{(L)}
\bigr),
\qquad
\widetilde B_i
=
\bigl(
\widetilde B_i^{(1)},\ldots,\widetilde B_i^{(L)}
\bigr).
\end{equation}
The empirical interval-wise objective is
\begin{equation}
\label{eq:app-practical-objective}
\widetilde{\mathcal L}_i(\omega_i\mid y;\vec\omega_i,\Psi)
=
\frac1L
\norm{
    \widetilde A_i-\omega_i\widetilde B_i
}_2^2 .
\end{equation}
The factor \(1/L\) does not affect the minimizer. The closed-form multi-test selector is
\begin{equation}
\label{eq:app-practical-multitest-selector}
\widehat\omega_i
=
    \frac{
        \ip{\widetilde A_i}{\widetilde B_i}
    }{
        \norm{\widetilde B_i}_2^2
        }
\end{equation}

\begin{algorithm}[t]
\caption{Practical interval-wise guidance selection}
\label{alg:practical-selector}
\begin{algorithmic}[1]
\STATE \textbf{Input:} condition \(y\); frozen fields
\(v_t^\theta(\cdot\mid y)\) and \(v_t^\theta(\cdot\mid\varnothing)\);
endpoint-conditioned field \(u_t(\cdot\mid x_1,y)\);
endpoint-conditioned density \(p_t(\cdot\mid x_1,y)\);
solver grid \(0=t_0<\cdots<t_T=1\);
rollout budget \(N\); endpoint budget \(M\);
test family \(\Psi=\{\psi_\ell\}_{\ell=1}^L\);
min and max constraint \(\omega_{\min},\omega_{\max}\);
mode \(\in\{\mathrm{online},\mathrm{offline}\}\);
random seeds.
\STATE Draw initial rollout particles
\begin{equation}
x_0^{(n)}\sim p_0,
\qquad
n=1,\ldots,N .
\end{equation}
\STATE Initialize the selected schedule as an empty list.
\FOR{\(i=0,\ldots,T-1\)}
    \STATE Set \(\Delta t_i=t_{i+1}-t_i\).
    \STATE Draw endpoint samples
    \begin{equation}
    x_{1,i}^{(m)}\sim q(\cdot\mid y),
    \qquad
    m=1,\ldots,M .
    \end{equation}
    \STATE Compute empirical posterior weights
    \(\{\pi_i^{(n,m)}\}_{n,m}\) using
    Eq.~\eqref{eq:app-practical-posterior-weights}.
    \STATE Compute the multi-test coefficients
    \(\widetilde A_i,\widetilde B_i\) using
    Eqs.~\eqref{eq:app-practical-A}-\eqref{eq:app-practical-B}.
    \STATE Select the interval guidance value using the stabilized quotient of
    Appendix~\ref{app:stabilization}, with \(\widetilde\beta_i=\norm{\widetilde B_i}_2^2\)
    and \(\widetilde\beta_i^\star=\max_{0\le j<i}\widetilde\beta_j\):
    \begin{equation}
    \widehat\omega_i
    \leftarrow
        \min\!\left\{
        \omega_{\max},
        \max\!\left\{\omega_{\min},\frac{
            \ip{\widetilde A_i}{\widetilde B_i}
        }{
            \max\!\left\{\widetilde\beta_i,\,\eta\,\widetilde\beta_i^\star\right\}
        }\right\}
    \right\}.
    \end{equation}
    \STATE Append \(\widehat\omega_i\) to the schedule.
    \STATE Advance all rollout particles by one solver step:
    \begin{equation}
    x_{i+1}^{(n)}
    \leftarrow
    \mathrm{ODEStep}
    \left(
        x_i^{(n)},t_i,t_{i+1};
        v^\theta(\cdot\mid y;\widehat\omega_i)
    \right),
    \qquad
    n=1,\ldots,N .
    \end{equation}
\ENDFOR
\IF{mode is \(\mathrm{online}\)}
    \STATE \textbf{Return:} selected schedule
    \((\widehat\omega_0,\ldots,\widehat\omega_{T-1})\)
    and generated particles \(\{x_T^{(n)}\}_{n=1}^{N}\).
\ELSE
    \STATE \textbf{Return:} fitted schedule
    \((\widehat\omega_0,\ldots,\widehat\omega_{T-1})\).
\ENDIF
\end{algorithmic}
\end{algorithm}

\paragraph{Offline reuse.}
In offline mode, Algorithm~\ref{alg:practical-selector} is run once using
fitting seeds and representative initial particles.  The resulting schedule is
stored as a piecewise-constant function on the solver grid.  Later inference
runs draw fresh initial particles and use the stored values
\((\widehat\omega_0,\ldots,\widehat\omega_{T-1})\) directly, without endpoint
sampling or posterior-weight computation.

\section{Conditional paths used for experimental validation}
\label{app:conditional-paths}

This appendix defines the endpoint-conditioned paths and velocity labels used for experimental validation.  Throughout, \(t\in[0,1]\) denotes generation time: \(t=0\)
corresponds to the Gaussian source and \(t=1\) corresponds to the conditional
endpoint law.  We do not use a separate diffusion time variable.  For a class
label \(y\), endpoint samples are drawn from \(x_1\sim q(\cdot\mid y)\), and
source samples are drawn from \(x_0\sim\mathcal N(0,\I)\).

\paragraph{Affine Gaussian path template.}
Several paths below are special cases of
\begin{equation}
x_t=a_t x_0+b_t x_1,
\qquad
x_0\sim\mathcal N(0,\I),
\label{eq:app-affine-gaussian-map}
\end{equation}
where \(a_t>0\) for \(t<1\).  The endpoint-conditioned law is
\begin{equation}
p_t(x\mid x_1,y)
=
\mathcal N\!\left(x\mid b_t x_1,a_t^2\I\right).
\label{eq:app-affine-gaussian-path}
\end{equation}
The associated endpoint-conditioned velocity field is
\begin{equation}
u_t(x\mid x_1,y)
=
\dot b_t x_1
+
\frac{\dot a_t}{a_t}
\left(x-b_t x_1\right),
\qquad t<1.
\label{eq:app-affine-gaussian-vf}
\end{equation}
Equivalently, along a sampled path \(x_t=a_t x_0+b_t x_1\),
\begin{equation}
u_t(x_t\mid x_1,y)=\dot a_t x_0+\dot b_t x_1 .
\end{equation}
The along-sample label is the quantity used whenever the paired source sample is
available.

\paragraph{Rectified flow (RF) \citep{liu2022flowstraightfastlearning}.}
The RF path uses the deterministic straight-line interpolation
\begin{equation}
x_t=(1-t)x_0+t x_1 .
\label{eq:app-rf-map}
\end{equation}
Thus \(a_t=1-t\) and \(b_t=t\).  The along-sample velocity label is
\begin{equation}
u_t(x_t\mid x_1,y)=x_1-x_0 .
\label{eq:app-rf-label}
\end{equation}

\paragraph{Optimal-transport Gaussian path (OT) \citep{lipman2023flowmatching}.}
The OT flow-matching path uses
\begin{equation}
x_t
=
\bigl(1-(1-\sigma_{\min})t\bigr)x_0+t x_1,
\qquad
\sigma_{\min}>0 .
\label{eq:app-ot-map}
\end{equation}
Thus
\begin{equation}
a_t=1-(1-\sigma_{\min})t,
\qquad
b_t=t.
\end{equation}
The endpoint-conditioned law is
\begin{equation}
p_t(x\mid x_1,y)
=
\mathcal N\!\left(
x\mid t x_1,
\bigl(1-(1-\sigma_{\min})t\bigr)^2\I
\right),
\label{eq:app-ot-path}
\end{equation}
and the explicit conditional velocity is
\begin{equation}
u_t(x\mid x_1,y)
=
\frac{
x_1-(1-\sigma_{\min})x
}{
1-(1-\sigma_{\min})t
}.
\label{eq:app-ot-vf}
\end{equation}
Along sampled paths, the velocity label is
\begin{equation}
u_t(x_t\mid x_1,y)=x_1-(1-\sigma_{\min})x_0 .
\label{eq:app-ot-label}
\end{equation}

\paragraph{Independent conditional flow matching (I-CFM) \citep{tong2024minibatchot}.}
For I-CFM, the conditioning variable is the pair
\begin{equation}
z=(x_0,x_1),
\qquad
x_0\sim\mathcal N(0,\I),
\qquad
x_1\sim q(\cdot\mid y),
\end{equation}
sampled independently.  The pair-conditioned path is
\begin{equation}
p_t(x\mid z,y)
=
\mathcal N\!\left(
x\mid (1-t)x_0+t x_1,\sigma_{\mathrm{icfm}}^2\I
\right),
\label{eq:app-icfm-path}
\end{equation}
or equivalently
\begin{equation}
x_t=(1-t)x_0+t x_1+\sigma_{\mathrm{icfm}}\varepsilon,
\qquad
\varepsilon\sim\mathcal N(0,\I).
\label{eq:app-icfm-map}
\end{equation}
Since the standard deviation is constant in time, the velocity label is
\begin{equation}
u_t(x\mid z,y)=x_1-x_0 .
\label{eq:app-icfm-vf}
\end{equation}

\paragraph{Variance-preserving probability-flow path (VP) \citep{song2021scorebasedsde}.}
For the VP variant, we write the signal coefficient directly in generation
time as \(\bar\alpha_t\), with \(\bar\alpha_0\approx0\) and
\(\bar\alpha_1=1\).  The endpoint-conditioned law is
\begin{equation}
p_t(x\mid x_1,y)
=
\mathcal N\!\left(
x\mid \bar\alpha_t x_1,
\bigl(1-\bar\alpha_t^2\bigr)\I
\right).
\label{eq:app-vp-path}
\end{equation}
Applying Eq.~\eqref{eq:app-affine-gaussian-vf} with
\(a_t=\sqrt{1-\bar\alpha_t^2}\) and \(b_t=\bar\alpha_t\) gives
\begin{equation}
u_t(x\mid x_1,y)
=
\frac{\dot{\bar\alpha}_t}
{1-\bar\alpha_t^2}
\left(
x_1-\bar\alpha_t x
\right),
\qquad t<1.
\label{eq:app-vp-vf}
\end{equation}
The VP field is evaluated only at solver times \(t_i<1\).  Endpoint quantities
at \(t=1\) are evaluated from the target law rather than by querying
Eq.~\eqref{eq:app-vp-vf}.

\begin{table}[t]
\caption{
Conditional paths used for experimental validation.  The time variable is always
generation time \(t\in[0,1]\), from prior to endpoint.}
\label{tab:app-conditional-paths}
\vspace{0.5em}
\centering
\begin{tabular}{llll}
\toprule
Variant & Conditioning & Path sample \(x_t\) & Velocity label \\
\midrule
RF \citep{liu2022flowstraightfastlearning}
&
\(x_1,y\)
&
\((1-t)x_0+t x_1\)
&
\(x_1-x_0\)
\\
OT \citep{lipman2023flowmatching}
&
\(x_1,y\)
&
\(\bigl(1-(1-\sigma_{\min})t\bigr)x_0+t x_1\)
&
\(x_1-(1-\sigma_{\min})x_0\)
\\
I-CFM \citep{tong2024minibatchot}
&
\(z=(x_0,x_1),y\)
&
\((1-t)x_0+t x_1+\sigma_{\mathrm{icfm}}\varepsilon\)
&
\(x_1-x_0\)
\\
VP \citep{song2021scorebasedsde}
&
\(x_1,y\)
&
\(\bar\alpha_t x_1+\sqrt{1-\bar\alpha_t^2}\,x_0\)
&
\(\partial_t x_t\)
\\
\bottomrule
\end{tabular}
\end{table}

\section{Experimental details}
\label{app:exp-details}

Here we provide the details for the experimental validation in Section~\ref{sec:experiments}.  All schedules are fitted
per class unless explicitly stated otherwise.  All methods in a comparison use
the same endpoint law, solver grid, initial latents, and inference seeds.
We report mean \(\pm\) sample standard deviation, where \(\mathrm{std} = \sqrt{\left(\sum_{i=1}^{|\mathrm{seeds}|}(\mathrm{metric}_i-\mathrm{mean})^2\right)/(|\mathrm{seeds}|-1)} \).

\subsection{Reproducibility summary}
\label{app:reproducibility-summary}

Unless otherwise stated, the experimental validation uses
\begin{equation}
T=200,\qquad
M=2^{14},\qquad
N=2^{14},\qquad
L=2^{12},\qquad
\omega_{\min}=1,\qquad
\omega_{\max}=\infty .
\end{equation}
The default test family contains \(2^{11}\) linear tests and
\(2^{11}\) quadratic tests.  We use three fitting seeds and three inference seeds, with matched prior latents across compared methods within each inference seed. The denominator stabilizer is specified in Appendix~\ref{app:stabilization}.

\begin{table}[t]
\caption{
Reproducibility parameters for the Gaussian-mixture experimental validation.}
\label{tab:app-reproducibility}
\vspace{0.5em}
\centering
\small
\begin{tabular}{ll}
\toprule
Item & Value / convention \\
\midrule
Endpoint law & Eq.~\eqref{eq:app-gom-endpoint-law} \\
Source law & \(\mathcal N(0,\I)\) \\
Dimension & \(d=2\) \\
Class schedules & fitted separately for each \(y\in\mathcal Y\) \\
Main solver & fixed-grid Euler \\
Solver grid & \(0=t_0<t_1<\cdots<t_T=1\) \\
Reference fitting grid & \(T=\) \\
Generated samples per inference seed & \(N_{\mathrm{gen}}\) \\
Reference samples per inference seed & \(N_{\mathrm{ref}}\) \\
Rollout particles for fitting & \(N\) \\
Endpoint samples for fitting & \(M\) \\
Number of test functions & \(L\) \\
Guidance clipping interval & \([\omega_{\min},\omega_{\max}]\) \\
Seed aggregation & mean \(\pm\) sample std over inference seeds \\
\bottomrule
\end{tabular}
\end{table}

Randomness is split into fitting randomness and inference randomness.  Fitting seeds determine the rollout particles, endpoint samples, and test-function draws used to estimate the schedule. Inference seeds determine the evaluation latents.  Within each inference seed and class, all compared methods receive the same initial latents \(x_0\).

For a piecewise-constant schedule \(\bar\omega_i\) on
\([t_i,t_{i+1})\), sampling uses
\begin{equation}
x_{i+1}
=
x_i
+
(t_{i+1}-t_i)
v_{t_i}^{\theta}(x_i\mid y;\bar\omega_i).
\label{eq:app-euler-update}
\end{equation}

For FM paths, the grid ends at \(t=1\).  For VP paths, the grid ends at
\(1-10^{-5}\).  All compared methods use matched initial latents within each
class and inference seed.

\subsection{Metric definitions}
\label{app:metric-definitions}

Let \(A=\{\hat x_i\}_{i=1}^{N_A}\) denote generated samples and
\(B=\{x_j\}_{j=1}^{N_B}\) denote reference samples.  Metrics are computed per
class and then averaged using the class prior \(\rho_y\).  For every reported
table or curve, the error bar is the sample standard deviation over
inference seeds.

\paragraph{Gaussian-fit KL.}
Let \((\hat\mu,\hat\Sigma)\) be the empirical mean and covariance of \(A\), and
let \((\mu,\Sigma)\) be the reference mean and covariance.  We report
\begin{equation}
\mathrm{KL}(A\|B)
=
\frac12
\left[
\operatorname{tr}(\Sigma^{-1}\hat\Sigma)
+
(\mu-\hat\mu)^\top\Sigma^{-1}(\mu-\hat\mu)
-d
+
\log\frac{\det\Sigma}{\det\hat\Sigma}
\right].
\label{eq:app-kl}
\end{equation}

\paragraph{Gaussian \(W_2\).}
We report the Gaussian Wasserstein distance
\begin{equation}
\mathrm{W}_2^2(A,B)
=
\norm{\hat\mu-\mu}_2^2
+
\operatorname{tr}
\left(
\hat\Sigma+\Sigma
-
2(\Sigma^{1/2}\hat\Sigma\Sigma^{1/2})^{1/2}
\right).
\label{eq:app-w2}
\end{equation}

\paragraph{MMD-RBF.}
Let
\(
\{\hat x^{(n)}\}_{n=1}^{N_{\mathrm{gen}}}
\)
denote the generated endpoint samples for a fixed inference seed, and let
\(
\{x^{(m)}\}_{m=1}^{N_{\mathrm{ref}}}
\)
denote the corresponding reference endpoint samples. We report
\begin{equation}
\mathrm{MMD}_{\mathrm{RBF}}^2
=
\frac{\sum_{n,n'=1}^{N_{\mathrm{gen}}}
k(\hat x^{(n)},\hat x^{(n')})}{N_{\mathrm{gen}}^2}
+
\frac{\sum_{m,m'=1}^{N_{\mathrm{ref}}}
k(x^{(m)},x^{(m')})}{N_{\mathrm{ref}}^2}
-
\frac{2\sum_{n=1}^{N_{\mathrm{gen}}}
\sum_{m=1}^{N_{\mathrm{ref}}}
k(\hat x^{(n)},x^{(m)})}{N_{\mathrm{gen}}N_{\mathrm{ref}}}
,
\label{eq:app-mmd}
\end{equation}
where \(k\) is the RBF kernel. The default kernel is the multi-bandwidth RBF kernel
\begin{equation}
k(u,v)
=
\frac13
\sum_{\alpha\in\{0.25,0.5,1.0\}}
\exp\left(
-\frac{\norm{u-v}_2^2}{2(\alpha\sigma_y)^2}
\right),
\label{eq:app-multibandwidth-kernel}
\end{equation}
where \(\sigma_y\) is computed once per class from reference samples using the
median heuristic.  The same bandwidth is used for all methods, guidance scales,
and inference seeds within a class.

\subsection{Data-generating law}
\label{app:data-law}

We use the analytic two-class Gaussian-mixture.  The endpoint law is
\begin{equation}
q(x_1)
=
\sum_{y\in\{0,1\}}\rho_y q(x_1\mid y),
\qquad
q(x_1\mid y)=\mathcal N(\mu_y,\Sigma_y),
\label{eq:app-gom-endpoint-law}
\end{equation}
with
\begin{equation}
\rho_0=\rho_1=\frac12,\qquad
\mu_0=(-2,0),\qquad
\mu_1=(2,0),\qquad
\Sigma_0=\Sigma_1=I_2 .
\end{equation}
The source law is \(p_0=\mathcal N(0,I_2)\), independent of \(y\).  Samples are
drawn directly from the analytic law. No saved dataset files or train/validation
splits are used.  Fitting and evaluation samples are generated with disjoint
random seeds.

\subsection{Flow variants and backbone}
\label{app:flow-backbone-details}

We evaluate the path variants defined in Appendix~\ref{app:conditional-paths}. I-CFM uses \(\sigma_{\mathrm{icfm}}=10^{-3}\). VP uses \(\beta_{\min}=0.1\), \(\beta_{\max}=20.0\), and endpoint truncation \(t_{\max}=1-10^{-5}\). All models were trained using implementations adapted from the torchcfm \citep{tong2024minibatchot} library; see Appendix~\ref{app:licenses}. We make only the minimal modification needed to pass the class label as an additional network input. All reported methods use frozen trained velocity fields. The backbone is a conditional MLP in \(d=2\), with hidden width \(64\), \(3\) layers, time embedding dimension \(2\), and class conditioning using one hot encoding \citep{potdar2017comparative} of the input label.
Conditional and unconditional predictions share one network trained with
null-label dropout probability \(p_{\mathrm{uncond}}=0.2\).  Training minimizes
the flow-matching regression loss \citep{zheng2023guidedflows, lipman2023flowmatching}
\(
\|v_\theta(t,x,y)-u_t(x\mid y)\|_2^2 .
\)
We use Adam optimizer \citep{kingma2015adam} with learning rate \(10^{-3}\), batch size \(256\), gradient clipping at norm \(1.0\), no EMA. All models are trained for \(20{,}000\) iterations.

The backbone for the MNIST is a conditional convolutional UNet on grayscale MNIST \citep{mnist} images \(x\in\mathbb{R}^{1\times 28\times 28}\) with base channel width \(64\), two down/up stages plus a middle block, sinusoidal time embedding of dimension \(128\) processed by a small time MLP, and class conditioning via a learned embedding of dimension \(32\) concatenated to the time features for FiLM-style modulation in residual blocks \citep{perez2018film}. We use two networks: a class-conditional UNet as above, and a second UNet with identical width and stage layout but no class input, trained separately to match the same flow target without conditioning. Training minimizes the flow-matching regression loss with the \(\ell_2\) norm taken over the pixel/channel dimensions of the velocity. All models are trained for \(100{,}000\) iterations.

\subsection{Baseline sweeps}
\label{app:baseline-sweeps}

All baseline comparisons use the same reference grid, matched latents, and inference seeds as the proposed schedule. For each method, the baseline values reported in the sweep tables correspond to the best configurations for that method. The best configuration is selected by the lowest average rank over the evaluated metrics; thus, only the best configurations for each method are included in these tables, with respect to the corresponding Cartesian product in Table~\ref{tab:app-baseline-sweeps}.

\begin{table}[t]
\caption{Baseline sweep grids.}
\label{tab:app-baseline-sweeps}
\vspace{0.5em}
\centering
\small
\begin{tabular}{ll}
\toprule
Method & Sweep grid \\
\midrule
Plain CFG \citep{zheng2023guidedflows}
&
\(\omega\in\{1.0,1.25,1.5,2.0,3.0\}\)
\\
CFG-Zero\({}^\ast\) \citep{fan2025cfgzerostar}
&
\(\omega\in\{1.0,1.25,1.5,2.0,3.0\}\), \(K_{\mathrm{zero}}\in\{1,2\}\)
\\
CFG-MP \citep{cai2026cfgmp}
&
\(\omega\in\{1.0,1.25,1.5,2.0,3.0\}\), \(K_{\mathrm{proj}}\in\{1,2\}\) (GM); \(K_{\mathrm{proj}}=3\) (MNIST)
\\
Rectified-CFG++ \citep{saini2025rectifiedcfgflowbasedmodels}
&
\(\lambda_{\max}\in\{0.5,1.0,2.0,3.0\}\), \(\sigma\in\{0.0,0.05\}\)
\\
\bottomrule
\end{tabular}
\end{table}

For the proposed method, every pair of fitting seed and test-function seed gives one fitted per-class schedule.  The Gaussian-mixture tables use the schedule fitted with the first fitting seed. For MNIST (Table~\ref{tab:mnist-baseline-comparison}) we apply to our method the same selection rule used for every baseline, reporting the fitting seed with the lowest mean FID over the three evaluation seeds; all three fitting seeds are listed in Table~\ref{tab:mnist-baseline-sweeps-t50}, and each of them is below the best baseline. Schedule-variance figures report variability over all fitting seeds.

\subsection{Test functions}
\label{app:test-functions}

The default weak-form objective uses polynomial test functions up to degree two,
\begin{equation}
\Psi_{\mathrm{poly},2}
=
\left\{
x\mapsto a^\top x
\right\}
\cup
\left\{
x\mapsto \frac12 x^\top A x
\right\},
\label{eq:app-poly-tests}
\end{equation}
so that the test family in Eq~\eqref{eq:multi-test-family} is formed by sampling $a \in \mathbb{R}^d$ and $A \in \mathbb{R}^{d \times d}$ from Gaussian ensembles $L$ times. To align with the theoretical requirement for compact support, each function is multiplied by a smooth $C^\infty$ bump $\eta_R(x)$ that vanishes outside a ball of radius $R+1$. In practice, we set $R=10^4$ so that $\eta_R \equiv 1$ for all samples; thus, the truncation is a mathematical formality that does not affect numerical results. The
same test-function family and seed convention are used for all compared methods
that require a fitted schedule. Figure~\ref{fig:test-function-sweep} isolates the effect of varying \(L\), while the ablation in Appendix~\ref{app:ablations} (Table~\ref{tab:ablation-main}) decomposes \(\Psi_{\mathrm{poly},2}\).

\begin{figure*}[h!]
    \centering
    \includegraphics[width=\linewidth]{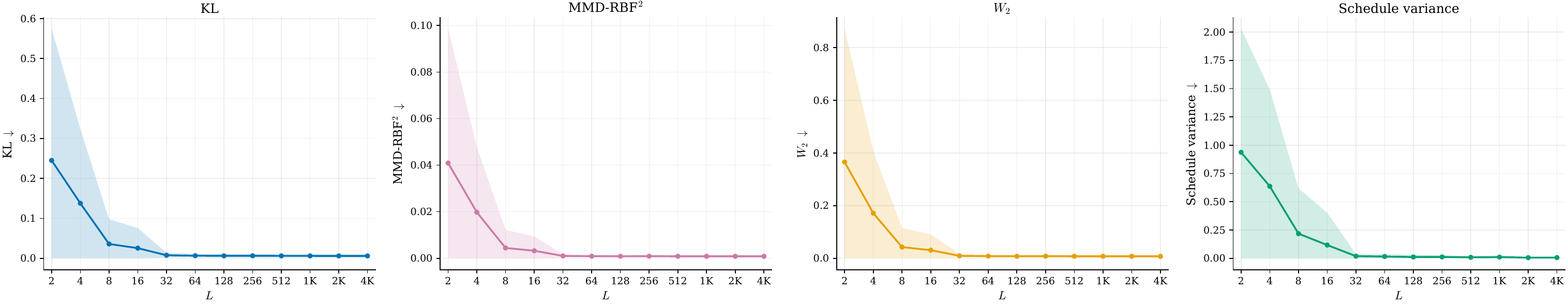}
    \caption{
    Sensitivity analysis of \(L\), the number of test functions used to estimate the selector. This graph aggregates over all flow variants, classes, fitting and inference seeds. Shaded regions show one standard deviation over fitting seeds. Downstream metrics and schedule variance stabilize at moderate \(L\).
    }
    \label{fig:test-function-sweep}
\end{figure*}

The family is interpretable and its blind spots are explicit: linear tests detect mean-velocity
mismatch, and quadratic tests additionally target the symmetric first spatial moment governing
second-moment evolution. In a Gaussian detectability check, injected mean and symmetric-linear
mismatches are detected \(2\)-\(4\) orders of magnitude above Monte Carlo error, whereas an
orthogonal third-Hermite mismatch \((x_1^3-3x_1)e_1\) is nearly invisible. Test functions should
therefore be chosen for the data modality: low-degree polynomials are appropriate where low-order
moments are informative, and richer families are needed for higher-order or structured discrepancies.

Sensitivity to the random probe draw is low. Holding the fitting latents, rollout noise, and inference
latents bit-identical, independent draws at \(L=2^{12}\) give mean schedule standard deviations
\(0.00033\) (linear), \(0.00207\) (mixed), and \(0.00660\) (quadratic-only), so the default mixed
family is highly reproducible.

Table~\ref{tab:probe-allocation} reallocates the same budget \(L=2^{12}\) across linear, quadratic,
and cubic probes on MNIST, paired against the default \((2048,2048,0)\) mixed family. All runs use
identical checkpoints, fitting particles, endpoint-sampling streams, solver grids, and evaluation
latents; only the probe allocation differs. Adding cubic probes yields at most modest gains while
keeping the fitted schedules close to the default, whereas an all-cubic family fails outright,
supporting the mixed default.

\begin{table}[h!]
\centering\small
\caption{Equal-budget probe allocations on MNIST. Paired \(\Delta\)FID is variant minus default; mean
schedule difference is the mean absolute scale difference over the \(50\) solver intervals.}
\label{tab:probe-allocation}
\vspace{0.5em}
\begin{tabular}{llccc}
\toprule
\((L_{\rm lin},L_{\rm quad},L_{\rm cubic})\) & Flow & Paired \(\Delta\)FID & 95\% CI & Mean schedule diff. \\
\midrule
\((2048,1024,1024)\) & RF & \(-0.0011\) & \([-0.0047,0.0038]\) & \(0.0033\pm0.0038\) \\
\((2048,1024,1024)\) & OT & \(-0.0220\) & \([-0.0321,-0.0133]\) & \(0.0045\pm0.0057\) \\
\((1024,2048,1024)\) & RF & \(-0.2499\) & \([-0.3860,-0.0626]\) & \(0.0103\pm0.0142\) \\
\((1024,2048,1024)\) & OT & \(-0.5663\) & \([-1.1632,0.1931]\) & \(0.0123\pm0.0156\) \\
\((1024,1024,2048)\) & RF & \(-0.2544\) & \([-0.3980,-0.0702]\) & \(0.0100\pm0.0131\) \\
\((1024,1024,2048)\) & OT & \(-0.5924\) & \([-1.1851,0.1669]\) & \(0.0144\pm0.0176\) \\
\((0,0,4096)\)       & RF & \(+64.8805\) & \([63.6820,66.1250]\) & \(1.7277\pm0.2654\) \\
\((0,0,4096)\)       & OT & \(+33.5259\) & \([31.8922,35.6230]\) & \(1.8590\pm0.7304\) \\
\bottomrule
\end{tabular}
\end{table}

\subsection{Numerical stabilization}
\label{app:stabilization}

The interval-wise selector solves a one-dimensional least-squares problem with
empirical coefficients \(\widetilde A_i,\widetilde B_i\).
The solution quotient~\eqref{eq:app-practical-multitest-selector} can become unstable when the empirical CFG direction is nearly
zero.  In that case, changing \(\omega_i\) has little effect on the local
objective, so the minimizer is poorly identified and small Monte Carlo errors in
\(\widetilde A_i\) may produce large jumps in the selected scale.

To stabilize only this numerical division, we define
\begin{equation}
    \widetilde\beta_i := \|\widetilde B_i\|_2^2,
    \qquad
    \widetilde\beta_i^\star := \max_{0\le j<i}\widetilde\beta_j,
\end{equation}
and replace the denominator by the running-max floor so that the implemented update is
\begin{equation}
    \widehat\omega_i
    =
    \frac{\ip{\widetilde A_i}{\widetilde B_i}}
    {\max\left\{
        \widetilde\beta_i,\,
        \eta\,\widetilde\beta_i^\star
    \right\}}
\end{equation}
In all experiments using the beta floor, we set
\(\eta= 10^{-2}\).  The floor is relative rather than absolute, so it
adapts to the scale of the fitted objective coefficients.  When inactive, the original closed-form selector is recovered exactly, when active, it only shrinks the raw quotient before projection.  The denominator-stabilization ablation in
Appendix~\ref{app:ablations} (Table~\ref{tab:ablation-main}) isolates the effect of varying \(\eta\).

Over the \(19{,}800\) selector updates in the reported runs, the floor is active in \(24.10\%\) of
updates and the projection returns the lower bound in \(53.34\%\). Floor activation is \(0\%\) on
MNIST, \(6.6\)-\(9.3\%\) on VP, and \(\approx 35\%\) on the remaining Gaussian-mixture flows.
Because \(\omega_{\min}=1\) is the conditional-only field, an interval projected to the bound applies
\emph{no} CFG extrapolation: on those intervals the selector determines that the conditional field
alone best matches the target transport. We adopt \(\omega_{\min}=1\) as an empirical stabilization
constraint; Appendix~\ref{app:omega-min} measures the effect of relaxing it.

\section{Ablation Study}
\label{app:ablations}

We evaluate the core components of our selector using a one-factor-at-a-time protocol across all flow variants. The default configuration uses $T=100$ intervals, $M=2^{14}$ endpoint samples, $N=2^{14}$ particles, and $L_{\mathrm{test}}=2^{12}$ mixed linear-quadratic test functions. Table~\ref{tab:ablation-main} summarizes the aggregate results over all flow variants. We ablate four primary axes to justify our design choices:

\begin{itemize}[leftmargin=*]
    \item \textbf{Schedule Reusability:} We compare \textit{online} selection against \textit{offline} reuse on fresh particles. This validates our method as an efficient "fit-once, use-many" procedure.
    \item \textbf{Endpoint Weighting:} We isolate the impact of \textit{posterior weighting} against naive uniform averaging, demonstrating that accounting for the endpoint posterior structure significantly improves guidance fidelity.
    \item \textbf{Denominator Stabilization:} We test the impact of the stability floor $\eta$. Table~\ref{tab:ablation-main} shows that while the floor has minimal effect in well-conditioned regimes, it is essential for preventing guidance spikes when the denominator vanishes.
    \item \textbf{Test-Function Family:} We evaluate linear, quadratic, and mixed families at a fixed computational budget. The mixed family provides the best balance, capturing both first- and second-moment mismatches.
\end{itemize}

Finally, we perform a sensitivity analysis on the number of test functions $L \in \{2^{i}\}_{i=1}^{12}$, visualized in Figure~\ref{fig:test-function-sweep}. We further analyze the sensitivity to the number of endpoint samples \(M\) and rollout particles \(N\) used by the schedule, with results shown in the heatmaps in Figure~\ref{fig:app-gm-mn-heatmaps}.

\begin{table*}[t]
  \caption{Aggregate ablation results over all flow variants. For each ablation variant and metric, we pool all raw measurements across the four flow backbones (RF, I-CFM, OT, and VP), fitting seeds, evaluation seeds, and classes. Each cell reports mean $\pm$ sample standard deviation. Lower is better.
}
  \label{tab:ablation-main}
  \centering
  \footnotesize
  \setlength{\tabcolsep}{3pt}
  \renewcommand{\arraystretch}{1.12}
  \begin{tabular}{llcccc}
    \toprule
    Ablation & Variant
    & \(\mathrm{KL}\downarrow\)
    & \(\mathrm{W}_2\downarrow\)
    & \(\mathrm{MMD}\downarrow\)
    & Time\(\downarrow\) \\
    \midrule
    reuse & online & 0.0059 $\pm$ 0.0041 & 0.0075 $\pm$ 0.0048 & 0.0009 $\pm$ 0.0005 & 13.58 $\pm$ 5.67 \\
     & offline, same $T$ & 0.0060 $\pm$ 0.0042 & 0.0076 $\pm$ 0.0049 & 0.0009 $\pm$ 0.0005 & 13.58 $\pm$ 5.66 \\
    \midrule
    weights & posterior & 0.0059 $\pm$ 0.0041 & 0.0074 $\pm$ 0.0048 & 0.0008 $\pm$ 0.0005 & 13.51 $\pm$ 5.72 \\
     & uniform & 0.0060 $\pm$ 0.0042 & 0.0075 $\pm$ 0.0049 & 0.0009 $\pm$ 0.0005 & 13.42 $\pm$ 5.61 \\
    \midrule
    stability & floor off & 0.0065 $\pm$ 0.0048 & 0.0079 $\pm$ 0.0053 & 0.0009 $\pm$ 0.0006 & 13.52 $\pm$ 5.73 \\
     & floor on, $\eta=10^{-2}$ & 0.0059 $\pm$ 0.0041 & 0.0075 $\pm$ 0.0048 & 0.0009 $\pm$ 0.0005 & 13.57 $\pm$ 5.60 \\
     & floor on, $\eta=10^{-3}$ & 0.0061 $\pm$ 0.0044 & 0.0075 $\pm$ 0.0049 & 0.0009 $\pm$ 0.0005 & 13.53 $\pm$ 5.63 \\
    \midrule
    tests & linear & 0.0064 $\pm$ 0.0041 & 0.0081 $\pm$ 0.0051 & 0.0009 $\pm$ 0.0005 & 13.25 $\pm$ 5.64 \\
     & quadratic & 0.3476 $\pm$ 0.3205 & 0.5253 $\pm$ 0.4899 & 0.0591 $\pm$ 0.0554 & 13.26 $\pm$ 5.63 \\
     & linear + quadratic & 0.0059 $\pm$ 0.0041 & 0.0074 $\pm$ 0.0048 & 0.0008 $\pm$ 0.0005 & 13.51 $\pm$ 5.72 \\
    \bottomrule
  \end{tabular}
\end{table*}

\begin{figure*}[h!]
    \centering
    \includegraphics[width=\linewidth]{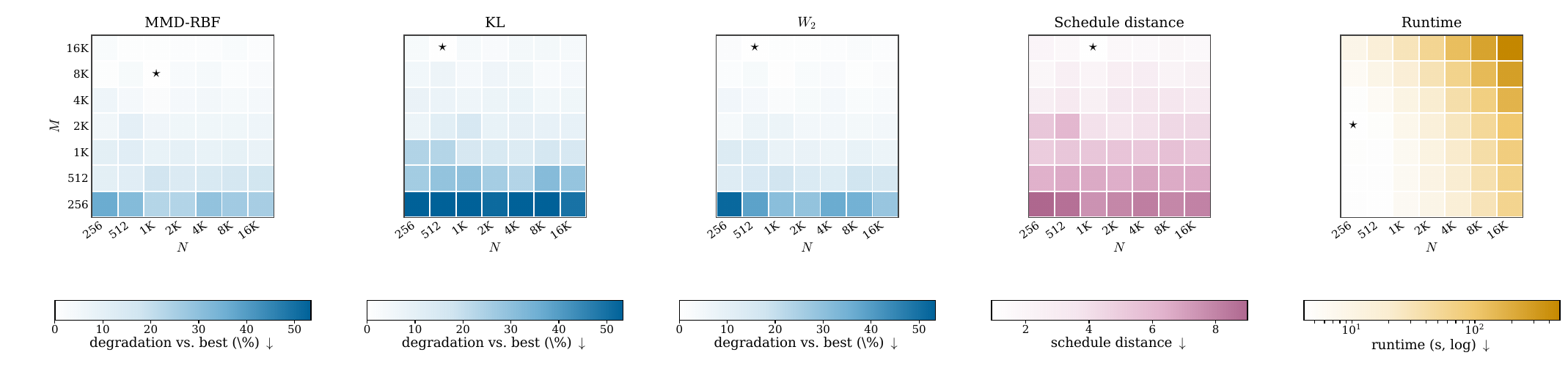}
    \caption{
    Full \((M,N)\) sweep for downstream metrics.
    Each cell shows the percentage gap from the best observed budget for that metric; lighter is better and the star marks the best cell.
    The heatmaps show that increasing \(M\) gives the main early improvement, while very large \(N\) has smaller marginal effect once \(M\) is moderate.
    }
    \label{fig:app-gm-mn-heatmaps}
\end{figure*}

\begin{figure*}[t]
    \centering

    \begin{minipage}[t]{0.9\linewidth}
        \centering
        \includegraphics[width=\linewidth]{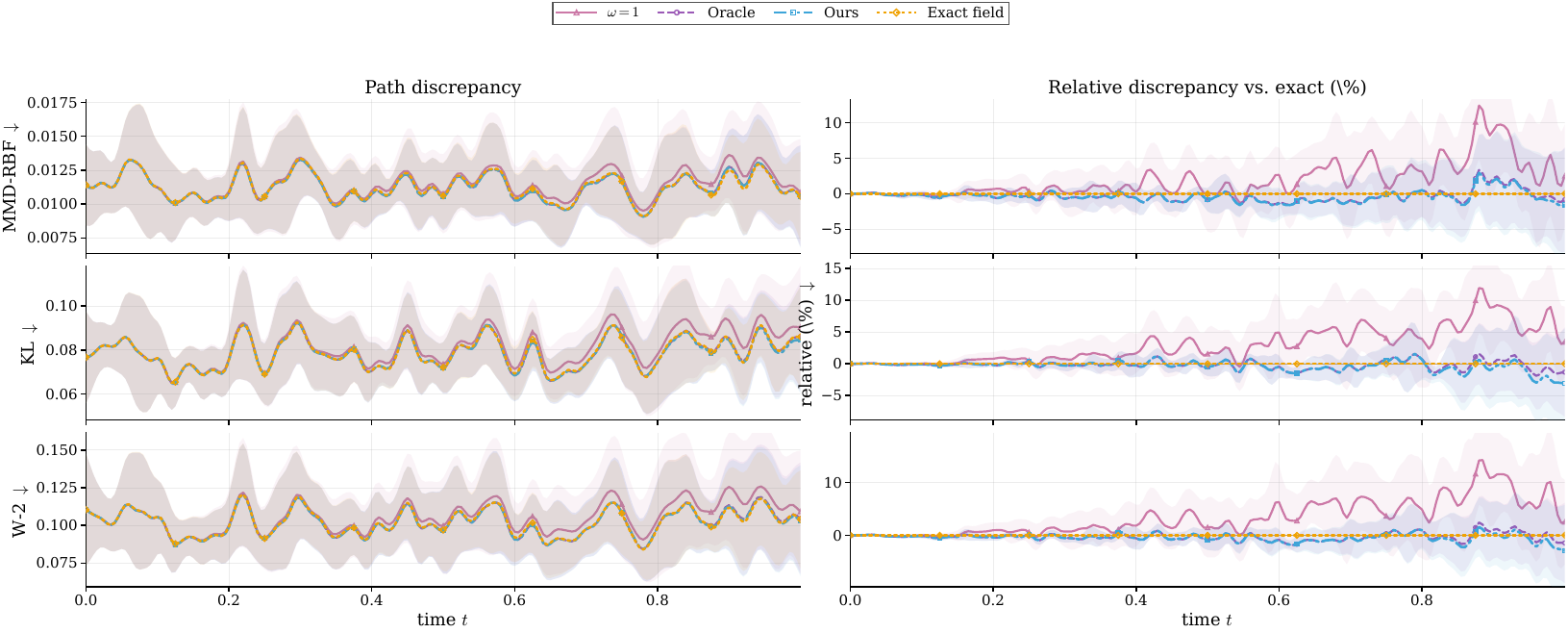}\\[-0.2em]
        (a) VP
    \end{minipage}

    \vspace{0.6em}

    \begin{minipage}[t]{0.9\linewidth}
        \centering
        \includegraphics[width=\linewidth]{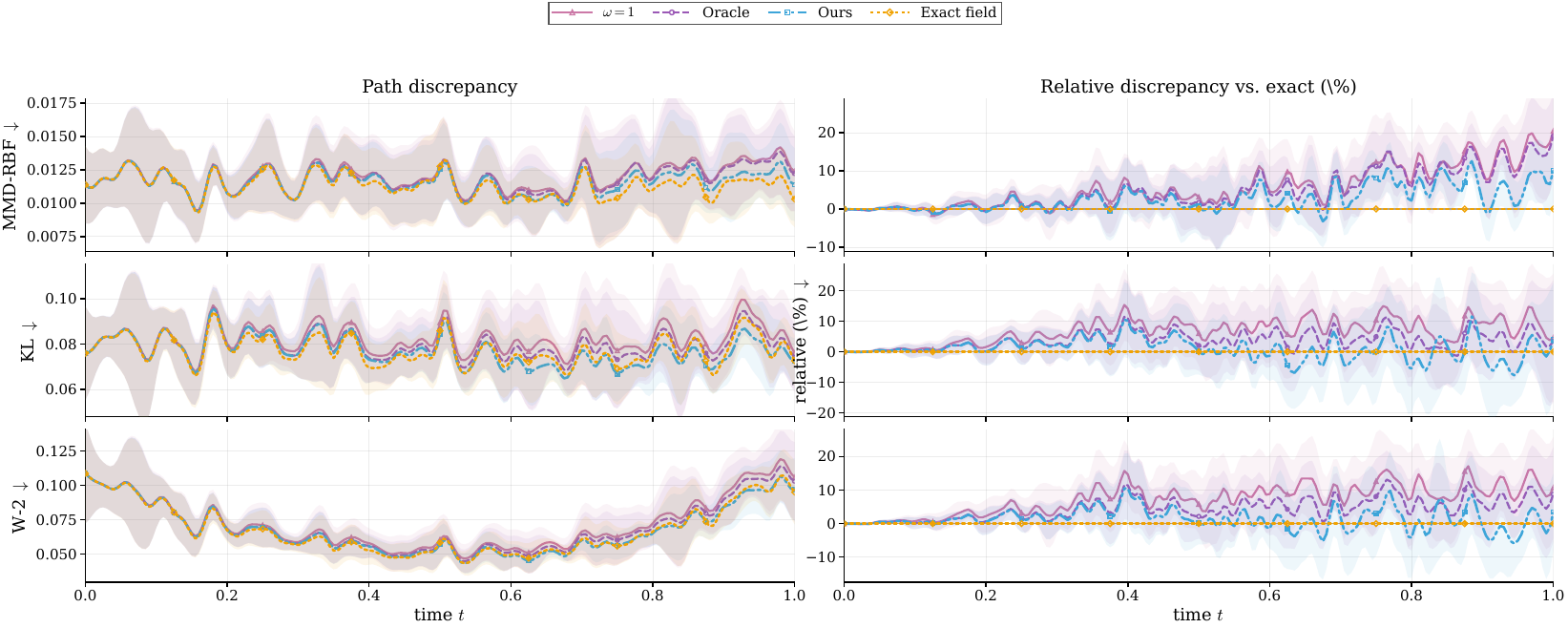}\\[-0.2em]
        (b) I-CFM
    \end{minipage}

    \vspace{0.35em}
    \caption{
    Online path discrepancy to the exact conditional marginal \(p_{t_i}(\cdot\mid y)\), measured along
    \(\hat p_{t_i}(\cdot\mid y;\vec\omega_i)\) by
    \(\mathrm{MMD}\text{-}\mathrm{RBF}^2\), \(\mathrm{KL}\), and \(\mathrm{W}_2\).
    Lower is better. Panel (a) shows VP, and panel (b) shows I-CFM.
    The true-velocity rollout provides a discretization reference. In both cases, the practical/oracle selector
    stays closer to this reference than the learned conditional rollout \(\omega\equiv1\), indicating
    improved path alignment.
    }
    \label{fig:online-path-alignment-vp-icfm}
\end{figure*}

\subsection{Does the finite residual track path discrepancy?}
\label{app:residual-diagnostic}
In Table~\ref{tab:residual-diagnostic} we empirically test the link between the approximated residual and the real path discrepancy. We hold
the interval \(t_i\), the rollout history, and the particle state fixed, vary only \(\omega\) over
\(26\) candidates in \([1,5]\), evaluate the held-out residual \(J_{r,i}(\omega)\), propagate one
matched Euler step, and measure the next-step discrepancy \(D_{r,i+1}(\omega)\). We report
\(\rho_{r,i}=\operatorname{Spearman}(J_{r,i}(\omega),D_{r,i+1}(\omega))\), averaged over
\(18\times20=360\) cells per flow, with 95\% CIs from a run-clustered bootstrap. 

\begin{table}[h!]
\centering\small
\caption{Correlation between the held-out finite residual and the resulting path discrepancy, at
fixed time and fixed rollout state. Mean Spearman correlation [95\% CI].}
\label{tab:residual-diagnostic}
\vspace{0.5em}
\begin{tabular}{lccccc}
\toprule
Flow & MMD & \(W_2\) & KL & Mean error & Cov.\ error \\
\midrule
RF    & \(0.57\,[0.49,0.65]\) & \(0.59\,[0.52,0.65]\) & \(0.55\,[0.48,0.62]\) & \(0.61\,[0.47,0.74]\) & \(0.47\,[0.40,0.54]\) \\
I-CFM & \(0.66\,[0.56,0.77]\) & \(0.65\,[0.55,0.74]\) & \(0.64\,[0.53,0.75]\) & \(0.62\,[0.47,0.76]\) & \(0.58\,[0.48,0.68]\) \\
OT    & \(0.50\,[0.38,0.61]\) & \(0.55\,[0.46,0.65]\) & \(0.56\,[0.45,0.66]\) & \(0.59\,[0.47,0.71]\) & \(0.41\,[0.30,0.53]\) \\
VP    & \(0.46\,[0.23,0.67]\) & \(0.50\,[0.28,0.70]\) & \(0.54\,[0.33,0.72]\) & \(0.16\,[-0.13,0.45]\) & \(0.44\,[0.26,0.60]\) \\
\bottomrule
\end{tabular}
\end{table}

\subsection{What the weak form contributes: pointwise projection control}
\label{app:pointwise-control}
Given the same estimated endpoint-conditioned information, seeds, and stabilization, we compare
PathGuide against selecting the scalar by a direct pointwise projection of the estimated conditional
velocity onto the CFG direction, See Table~\ref{tab:pointwise-control}.
With identical endpoint information, the on-policy weak-form selector wins on RF, I-CFM, and OT
across all three metrics and ties on VP. The gain is therefore attributable to the criterion, not to
the endpoint-conditioned estimate itself.

\begin{table}[h!]
\centering\small
\caption{Pointwise projection versus the weak-form selector under identical endpoint information. Mean \(\pm\) std over 18 paired settings; lower is better.}
\label{tab:pointwise-control}
\vspace{0.5em}
\setlength{\tabcolsep}{4pt}
\begin{tabular}{lcc}
\toprule
Flow & Pointwise projection: \(W_2^2\) / KL / MMD\(^2\) & PathGuide: \(W_2^2\) / KL / MMD\(^2\) \\
\midrule
RF    & \(.0592{\pm}.0034\) / \(.0444{\pm}.0050\) / \(.00532{\pm}.00071\) & \(.0135{\pm}.0056\) / \(.0109{\pm}.0040\) / \(.00180{\pm}.00065\) \\
I-CFM & \(.0580{\pm}.0034\) / \(.0440{\pm}.0045\) / \(.00539{\pm}.00069\) & \(.0130{\pm}.0055\) / \(.0108{\pm}.0039\) / \(.00178{\pm}.00065\) \\
OT    & \(.0441{\pm}.0214\) / \(.0329{\pm}.0178\) / \(.00433{\pm}.00134\) & \(.0069{\pm}.0046\) / \(.0053{\pm}.0028\) / \(.00100{\pm}.00040\) \\
VP    & \(.0056{\pm}.0031\) / \(.0040{\pm}.0022\) / \(.00066{\pm}.00019\) & \(.0058{\pm}.0031\) / \(.0040{\pm}.0024\) / \(.00067{\pm}.00020\) \\
\bottomrule
\end{tabular}
\end{table}

\subsection{Relaxing the guidance lower bound}
\label{app:omega-min}
We rerun the complete fitting procedure with the admissible lower bound relaxed from
\(\omega_{\min}=1\) to \(0\), using the same backbones, fitting and evaluation seeds, latents, solver
grids, and estimator budgets. In Table~\ref{tab:omega-min-relax} we report \(\Delta=\text{relaxed}-\text{default}\), so positive values
indicate worse performance under relaxation. Results average all matched fitting/evaluation-seed
pairs.

\begin{table}[h!]
\centering\small
\caption{Effect of relaxing the admissible lower bound from \(\omega_{\min}=1\) to \(0\). Positive value means \(\omega_{\min}=1\) is better.}
\label{tab:omega-min-relax}
\vspace{0.5em}
\begin{tabular}{lccc}
\toprule
Flow & \(\Delta\)KL \(T{=}200/500\) & \(\Delta W_2\) \(T{=}200/500\) & \(\Delta\)MMD\(^2\) \(T{=}200/500\) \\
\midrule
RF    & \(+.03264/+.03446\) & \(+.04097/+.04310\) & \(+.001005/+.001104\) \\
I-CFM & \(+.01161/+.01247\) & \(+.01303/+.01397\) & \(+.000593/+.000650\) \\
OT    & \(+.00950/+.01033\) & \(+.01060/+.01148\) & \(+.000471/+.000528\) \\
VP    & \(+.000001/+.000006\) & \(+.000139/+.000145\) & \(+.000004/+.000004\) \\
\bottomrule
\end{tabular}

\vspace{0.6em}
\begin{tabular}{lc}
\toprule
MNIST flow & \(\Delta\)FID [95\% CI] \\
\midrule
RF & \(-0.940\,[-1.028,-0.853]\) \\
OT & \(+1.408\,[1.316,1.501]\) \\
\bottomrule
\end{tabular}
\end{table}

\begin{figure*}[t!]
    \centering
    \includegraphics[width=\linewidth]{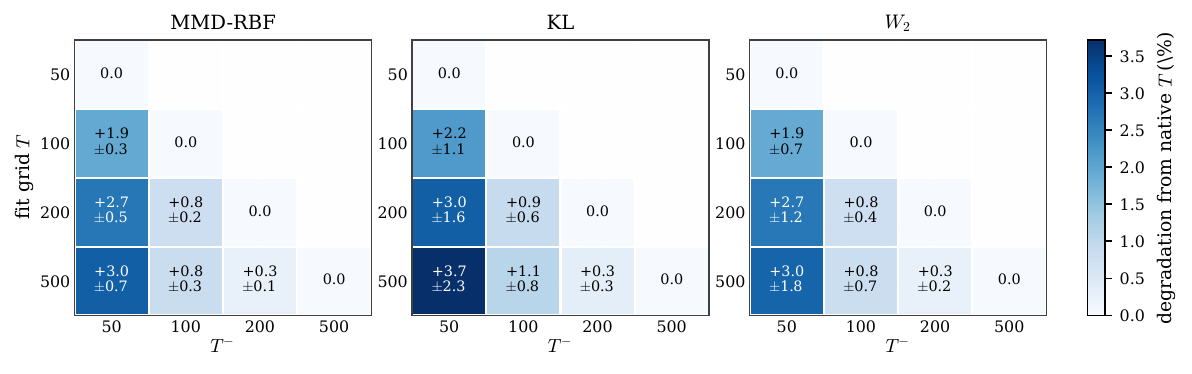}
    \caption{
    Relative degradation under coarser inference.
    Each entry reports the percentage increase in the metric when a schedule fitted at \(T\) is deployed at \(T^{-}\), relative to evaluating the same schedule at its native grid \(T\).
    Lower is better; \(0\%\) corresponds to native-grid evaluation.
    }
    \label{fig:app-gm-infer-at-tminus-degradation}
\end{figure*}

\subsection{Computational cost}
\label{app:computational-cost}

The method has two modes.  In online selection, objective coefficients are
estimated during sampling.  In offline schedule fitting, the schedule is fitted
once for a fixed backbone, class, solver, and test-function family, then
reused.  The main experimental validation uses the offline mode unless explicitly stated.

Let \(C_{\mathrm{cond}}\) and \(C_{\mathrm{uncond}}\) denote the cost of one conditional and unconditional backbone evaluation, and let \(C_{\mathrm{path}}\) denote the cost of evaluating endpoint-conditioned path densities and velocity labels for posterior weighting.  For \(T\) solver intervals, \(N\) rollout particles, \(M\) endpoint samples, and \(L\) test functions, the dominant offline fitting cost scales as
\begin{equation}
\mathcal O\!\left(
T N
\left[
C_{\mathrm{cond}}+C_{\mathrm{uncond}}
+
M C_{\mathrm{path}}
+
L d
\right]
\right).
\label{eq:app-fit-cost}
\end{equation}
After a schedule is fitted, deployment has the same backbone-evaluation cost as
standard CFG on the same solver grid:
\begin{equation}
\mathcal O\!\left(
T N_{\mathrm{gen}}
\left[
C_{\mathrm{cond}}+C_{\mathrm{uncond}}
\right]
\right).
\label{eq:app-deploy-cost}
\end{equation}

\begin{table}[h!]
\caption{
Cost accounting for fitted schedules.  Offline fitting is amortized across
future deployments; deployment uses the stored schedule and therefore has the
same NFE-equivalent as standard CFG.}
\label{tab:app-computational-cost}
\vspace{0.5em}
\centering
\small
\begin{tabular}{lllll}
\toprule
Method & Offline fit cost & Sampling cost & Memory & NFE-equivalent \\
\midrule
Conditional-only
&
none
&
\(T C_{\mathrm{cond}}\)
&
\(\mathcal O(N_{\mathrm{gen}}d)\)
&
\(T\)
\\
Constant CFG
&
sweep over \(\Omega_{\mathrm{CFG}}\)
&
\(T(C_{\mathrm{cond}}+C_{\mathrm{uncond}})\)
&
\(\mathcal O(N_{\mathrm{gen}}d)\)
&
\(2T\)
\\
Proposed, offline
&
Eq.~\eqref{eq:app-fit-cost}
&
\(T(C_{\mathrm{cond}}+C_{\mathrm{uncond}})\)
&
\(\mathcal O((N+M)d)\)
&
\(2T\)
\\
Proposed, online
&
none
&
Eq.~\eqref{eq:app-fit-cost} during sampling
&
\(\mathcal O((N+M)d)\)
&
\(2T\) plus estimator work
\\
\bottomrule
\end{tabular}
\end{table}

\begin{table}[h!]
\centering
\caption{Measured cost per workload. Each cell reports wall-clock time / peak memory / backbone NFE.
Deployment with a stored schedule matches plain CFG on all three; only fitting is additional, and it
is amortized across later deployments.}
\label{tab:app-measured-cost}
\vspace{0.5em}
\small
\setlength{\tabcolsep}{5pt}
\begin{tabular}{lccc}
\toprule
Workload & \(\omega\equiv1\) & CFG / stored PathGuide & PathGuide fitting \\
\midrule
GMM, \(N=M=2^{14}\), \(L=2^{12}\), \(T=500\)
& 10.0\,s / 34\,MB / \(T\)
& 19.7\,s / 34\,MB / \(2T\)
& 137\,s / 333\,MB / \(2T\) \\
MNIST, \(N=2^{12}\), \(M=2^{11}\), \(L=2^{12}\), \(T=50\)
& 29.9\,s / 315\,MB / \(T\)
& 59.6\,s / 327\,MB / \(2T\)
& 230\,s / 578\,MB / \(2T\) \\
\bottomrule
\end{tabular}
\end{table}

The estimator reuses the conditional and unconditional backbone evaluations already required by the
ODE step, so it adds no NFE; its overhead is posterior and test-function estimation. One fit costs
\(\approx 7\) GMM or \(\approx 4\) MNIST CFG generations, and amortizes against repeated online
selection after two stored-schedule deployments.

\paragraph{Experiments compute resources.}
All experiments were run on a single workstation equipped with six NVIDIA RTX~2080~Ti GPUs,
using CUDA~12.4, PyTorch~2.6.0+cu124, and Python~3.9.20. Table~\ref{tab:compute-resources}
summarizes the aggregate compute budget across all completed experiment runs used in the paper.

\begin{table}[h!]
\centering
\caption{Compute resources for the experiment families reported in the paper. All experiments were run on the same workstation. Runtime is reported as accumulated single-GPU time.}
\label{tab:compute-resources}
\vspace{0.5em}
\begin{tabular}{lrrr}
\toprule
Experiment family & Runs & Total time & Mean time \\
\midrule
  Baseline sweeps & 163 & 226.9h & 1.4h \\
  Correction-field diagnostics & 173 & 27.8h & 9.6min \\
  Online path alignment & 11 & 22.8h & 2.1h \\
  Fit-at-\(T\), infer-at-coarser-\(T\) & 136 & 145.3h & 1.1h \\
  Monte Carlo scaling & 87 & 236.5h & 2.7h \\
  Schedule heatmap overlays & 86 & 60.3h & 42.1min \\
  Fitting-resolution sweep & 81 & 70.9h & 52.5min \\
  MNIST velocity training & 2 & 8.8h & 4.4h \\
  Gaussian Mixture velocity training & 4 & 8.0min & 2.0min \\
\midrule
  \textbf{Total reported} & \textbf{743} & \textbf{799.4h} & \textbf{1.1h} \\
\bottomrule
\end{tabular}
\vspace{0.75em}
\end{table}

\section{Licenses}
\label{app:licenses}

This project utilizes the following third-party libraries and datasets:

\begin{itemize}
    \item \textbf{torchcfm} (v1.0.7): \citet{tong2023simulation, tong2024minibatchot}, MIT License. Used for flow matching and OT-CFM training. \url{https://github.com}
    \item \textbf{MNIST Dataset}: \citet{mnist}, CC BY-SA 3.0 License. Used for image experiments. \url{http://lecun.com}
    \item \textbf{PyTorch}: \citet{pytorch}, BSD-3-Clause License. Base framework. \url{https://pytorch.org}
    \item \textbf{POT (Python Optimal Transport)}: \citet{flamary2021pot}, MIT License. Used for $W_2$ computation. \url{https://github.io}
    \item \textbf{torch-fidelity}: \citet{obukhov2020torchfidelity}, Apache-2.0 License. Used for FID computation. \url{https://github.com}
\end{itemize}

\section{Full tables}
\label{app:baseline-sweeps-tables}

\begin{table*}[t]
    \caption{
    Appendix sweep results for baseline configurations at \(T=200\).
    For each flow variant and method, we report up to three configurations from the hyperparameter sweep.
    Lower is better for all metrics.
    All entries are evaluated over 3 inference seeds using \(2^{14}\) generated samples per seed, and report mean \(\pm\) standard deviation.
    The mean is averaged over classes per evaluation seed; the standard deviation is the sample std over the 3 per-seed averages.
    }
    \label{tab:baseline-sweeps-t200}
    \vspace{0.5em}
    \centering
    \scriptsize
    \setlength{\tabcolsep}{3pt}
    \renewcommand{\arraystretch}{1.12}
    \begin{tabular}{@{}llp{0.32\textwidth}ccc@{}}
      \toprule
      Variant & Method & Configuration
        & KL\(\downarrow\)
        & \(W_2\downarrow\)
        & MMD\(\downarrow\) \\
      \midrule

      RF & Plain CFG
        & \(\omega=1\)
        & \(0.0044_{0.0013}\) & \(0.0060_{0.0016}\) & \(0.0006_{0.0001}\) \\
      & & \(\omega=1.25\)
        & \(0.0441_{0.0005}\) & \(0.0657_{0.0025}\) & \(0.0054_{0.0003}\) \\
      & & \(\omega=1.5\)
        & \(0.1298_{0.0015}\) & \(0.1961_{0.0052}\) & \(0.0175_{0.0005}\) \\

      & CFG-Zero\(^*\)
        & \(\omega=1\), \(K_{\mathrm{zero}}=1\)
        & \(0.0046_{0.0013}\) & \(0.0064_{0.0018}\) & \(0.0006_{0.0001}\) \\
      & & \(\omega=1\), \(K_{\mathrm{zero}}=2\)
        & \(0.0050_{0.0014}\) & \(0.0072_{0.0019}\) & \(0.0006_{0.0001}\) \\
      & & \(\omega=1.25\), \(K_{\mathrm{zero}}=2\)
        & \(0.0080_{0.0005}\) & \(0.0127_{0.0013}\) & \(0.0013_{0.0002}\) \\

      & CFG-MP
        & \(\omega=3\), \(K_{\mathrm{proj}}=2\)
        & \(6.2232_{0.0225}\) & \(9.5319_{0.0356}\) & \(0.6234_{0.0018}\) \\
      & & \(\omega=2\), \(K_{\mathrm{proj}}=2\)
        & \(6.1985_{0.0204}\) & \(9.5446_{0.0354}\) & \(0.6355_{0.0018}\) \\
      & & \(\omega=1.25\), \(K_{\mathrm{proj}}=2\)
        & \(6.0096_{0.2804}\) & \(9.5237_{0.1140}\) & \(0.6443_{0.0016}\) \\

      & R-CFG++
        & \(\lambda=0.5\), \(\sigma=0\)
        & \(0.0191_{0.0006}\) & \(0.0248_{0.0006}\) & \(0.0020_{0.0001}\) \\
      & & \(\lambda=0.5\), \(\sigma=0.05\)
        & \(0.0192_{0.0006}\) & \(0.0248_{0.0006}\) & \(0.0020_{0.0001}\) \\
      & & \(\lambda=1\), \(\sigma=0\)
        & \(0.0534_{0.0003}\) & \(0.0688_{0.0017}\) & \(0.0057_{0.0002}\) \\
\cmidrule(lr){3-6}
      & PathGuide (Ours)
        & \(\mathrm{seed}_1\)
        & \boldmath{\(0.0039_{0.0012}\)} & \boldmath{\(0.0056_{0.0013}\)} & \boldmath{\(0.0005_{0.0001}\)} \\
      & & \(\mathrm{seed}_2\)
        & \(0.0039_{0.0011}\) & \(0.0056_{0.0012}\) & \(0.0005_{0.0001}\) \\
      & & \(\mathrm{seed}_3\)
        & \(0.0040_{0.0011}\) & \(0.0058_{0.0011}\) & \(0.0005_{0.0001}\) \\

      \midrule

      I-CFM & Plain CFG
        & \(\omega=1\)
        & \(0.0084_{0.0021}\) & \(0.0112_{0.0029}\) & \(0.0012_{0.0003}\) \\
      & & \(\omega=1.25\)
        & \(0.0423_{0.0004}\) & \(0.0645_{0.0018}\) & \(0.0056_{0.0002}\) \\
      & & \(\omega=1.5\)
        & \(0.1220_{0.0008}\) & \(0.1898_{0.0043}\) & \(0.0172_{0.0004}\) \\

      & CFG-Zero\(^*\)
        & \(\omega=1\), \(K_{\mathrm{zero}}=1\)
        & \(0.0084_{0.0022}\) & \(0.0114_{0.0030}\) & \(0.0012_{0.0003}\) \\
      & & \(\omega=1\), \(K_{\mathrm{zero}}=2\)
        & \(0.0086_{0.0022}\) & \(0.0119_{0.0031}\) & \(0.0012_{0.0003}\) \\
      & & \(\omega=1.25\), \(K_{\mathrm{zero}}=2\)
        & \(0.0101_{0.0010}\) & \(0.0157_{0.0016}\) & \(0.0018_{0.0003}\) \\

      & CFG-MP
        & \(\omega=1\), \(K_{\mathrm{proj}}=1\)
        & \(0.1197_{0.0007}\) & \(0.1853_{0.0042}\) & \(0.0168_{0.0004}\) \\
      & & \(\omega=1.25\), \(K_{\mathrm{proj}}=1\)
        & \(0.2140_{0.0016}\) & \(0.3348_{0.0062}\) & \(0.0322_{0.0006}\) \\
      & & \(\omega=1.5\), \(K_{\mathrm{proj}}=1\)
        & \(0.3133_{0.0024}\) & \(0.4939_{0.0080}\) & \(0.0497_{0.0007}\) \\

      & R-CFG++
        & \(\lambda=0.5\), \(\sigma=0\)
        & \(0.0192_{0.0013}\) & \(0.0252_{0.0011}\) & \(0.0024_{0.0002}\) \\
      & & \(\lambda=0.5\), \(\sigma=0.05\)
        & \(0.0193_{0.0013}\) & \(0.0252_{0.0011}\) & \(0.0024_{0.0002}\) \\
      & & \(\lambda=1\), \(\sigma=0\)
        & \(0.0500_{0.0008}\) & \(0.0652_{0.0011}\) & \(0.0057_{0.0002}\) \\
\cmidrule(lr){3-6}
      & PathGuide (Ours)
        & \(\mathrm{seed}_1\)
        & \boldmath{\(0.0076_{0.0019}\)} & \boldmath{\(0.0099_{0.0022}\)} & \boldmath{\(0.0012_{0.0003}\)} \\
      & & \(\mathrm{seed}_2\)
        & \(0.0077_{0.0018}\) & \(0.0100_{0.0021}\) & \(0.0012_{0.0003}\) \\
      & & \(\mathrm{seed}_3\)
        & \(0.0078_{0.0018}\) & \(0.0102_{0.0020}\) & \(0.0012_{0.0003}\) \\

      \midrule

      OT & Plain CFG
        & \(\omega=1\)
        & \(0.0082_{0.0020}\) & \(0.0111_{0.0027}\) & \(0.0012_{0.0003}\) \\
      & & \(\omega=1.25\)
        & \(0.0434_{0.0003}\) & \(0.0669_{0.0020}\) & \(0.0058_{0.0002}\) \\
      & & \(\omega=1.5\)
        & \(0.1238_{0.0010}\) & \(0.1935_{0.0045}\) & \(0.0175_{0.0004}\) \\

      & CFG-Zero\(^*\)
        & \(\omega=1\), \(K_{\mathrm{zero}}=1\)
        & \(0.0081_{0.0020}\) & \(0.0112_{0.0029}\) & \boldmath{\(0.0012_{0.0003}\)} \\
      & & \(\omega=1\), \(K_{\mathrm{zero}}=2\)
        & \(0.0082_{0.0021}\) & \(0.0115_{0.0030}\) & \(0.0012_{0.0003}\) \\
      & & \(\omega=1.25\), \(K_{\mathrm{zero}}=2\)
        & \(0.0105_{0.0009}\) & \(0.0169_{0.0016}\) & \(0.0018_{0.0003}\) \\

      & CFG-MP
        & \(\omega=1\), \(K_{\mathrm{proj}}=1\)
        & \(0.1215_{0.0009}\) & \(0.1889_{0.0043}\) & \(0.0171_{0.0004}\) \\
      & & \(\omega=1.25\), \(K_{\mathrm{proj}}=1\)
        & \(0.2155_{0.0018}\) & \(0.3387_{0.0064}\) & \(0.0325_{0.0006}\) \\
      & & \(\omega=1.5\), \(K_{\mathrm{proj}}=1\)
        & \(0.3137_{0.0026}\) & \(0.4972_{0.0082}\) & \(0.0498_{0.0007}\) \\

      & R-CFG++
        & \(\lambda=0.5\), \(\sigma=0\)
        & \(0.0199_{0.0012}\) & \(0.0266_{0.0011}\) & \(0.0024_{0.0002}\) \\
      & & \(\lambda=0.5\), \(\sigma=0.05\)
        & \(0.0200_{0.0012}\) & \(0.0266_{0.0011}\) & \(0.0025_{0.0002}\) \\
      & & \(\lambda=1\), \(\sigma=0\)
        & \(0.0515_{0.0007}\) & \(0.0676_{0.0012}\) & \(0.0060_{0.0002}\) \\
\cmidrule(lr){3-6}
      & PathGuide (Ours)
        & \(\mathrm{seed}_1\)
        & \boldmath{\(0.0076_{0.0017}\)} & \boldmath{\(0.0102_{0.0021}\)} & \(0.0012_{0.0003}\) \\
      & & \(\mathrm{seed}_2\)
        & \(0.0077_{0.0017}\) & \(0.0103_{0.0020}\) & \(0.0012_{0.0003}\) \\
      & & \(\mathrm{seed}_3\)
        & \(0.0078_{0.0017}\) & \(0.0105_{0.0020}\) & \(0.0012_{0.0003}\) \\

      \midrule

      VP & Plain CFG
        & \(\omega=1\)
        & \(0.0030_{0.0007}\) & \(0.0052_{0.0013}\) & \(0.0005_{0.0001}\) \\
      & & \(\omega=1.25\)
        & \(0.0254_{0.0009}\) & \(0.0372_{0.0027}\) & \(0.0027_{0.0002}\) \\
      & & \(\omega=1.5\)
        & \(0.0912_{0.0018}\) & \(0.1403_{0.0052}\) & \(0.0116_{0.0004}\) \\

      & CFG-Zero\(^*\)
        & \(\omega=1.5\), \(K_{\mathrm{zero}}=1\)
        & \(0.0013_{0.0003}\) & \(0.0019_{0.0005}\) & \(0.0003_{0.0000}\) \\
      & & \(\omega=1.5\), \(K_{\mathrm{zero}}=2\)
        & \(0.0014_{0.0003}\) & \(0.0020_{0.0005}\) & \(0.0003_{0.0000}\) \\
      & & \(\omega=1.25\), \(K_{\mathrm{zero}}=1\)
        & \(0.0019_{0.0004}\) & \(0.0031_{0.0008}\) & \(0.0003_{0.0000}\) \\

      & CFG-MP
        & \(\omega=1\), \(K_{\mathrm{proj}}=2\)
        & \(2.9991_{0.0226}\) & \(4.7929_{0.0130}\) & \(0.4827_{0.0008}\) \\
      & & \(\omega=1.25\), \(K_{\mathrm{proj}}=2\)
        & \(3.1150_{0.0695}\) & \(4.9108_{0.0261}\) & \(0.4866_{0.0008}\) \\
      & & \(\omega=1.5\), \(K_{\mathrm{proj}}=2\)
        & \(3.1485_{0.0620}\) & \(4.9986_{0.0235}\) & \(0.4905_{0.0008}\) \\

      & R-CFG++
        & \(\lambda=0.5\), \(\sigma=0\)
        & \(0.0290_{0.0007}\) & \(0.0388_{0.0023}\) & \(0.0029_{0.0002}\) \\
      & & \(\lambda=0.5\), \(\sigma=0.05\)
        & \(0.0291_{0.0007}\) & \(0.0389_{0.0023}\) & \(0.0029_{0.0002}\) \\
      & & \(\lambda=1\), \(\sigma=0\)
        & \(0.0931_{0.0013}\) & \(0.1318_{0.0045}\) & \(0.0110_{0.0004}\) \\
\cmidrule(lr){3-6}
      & PathGuide (Ours)
        & \(\mathrm{seed}_1\)
        & \boldmath{\(0.0011_{0.0004}\)} & \(0.0015_{0.0005}\) & \boldmath{\(0.0003_{0.0000}\)} \\
      & & \(\mathrm{seed}_2\)
        & \(0.0011_{0.0004}\) & \boldmath{\(0.0015_{0.0005}\)} & \(0.0003_{0.0000}\) \\
      & & \(\mathrm{seed}_3\)
        & \(0.0011_{0.0004}\) & \(0.0015_{0.0005}\) & \(0.0003_{0.0000}\) \\
      \bottomrule
    \end{tabular}
\end{table*}

\begin{table*}[t]
    \caption{
    Appendix sweep results for baseline configurations at \(T=500\).
    For each flow variant and method, we report up to three configurations from the hyperparameter sweep.
    Lower is better for all metrics.
    All entries are evaluated over 3 inference seeds using \(2^{14}\) generated samples per seed, and report mean \(\pm\) standard deviation.
    The mean is averaged over classes per evaluation seed; the standard deviation is the sample std over the 3 per-seed averages.
    }
    \label{tab:baseline-sweeps-t500}
    \vspace{0.5em}
    \centering
    \scriptsize
    \setlength{\tabcolsep}{3pt}
    \renewcommand{\arraystretch}{1.12}
    \begin{tabular}{@{}llp{0.32\textwidth}ccc@{}}
      \toprule
      Variant & Method & Configuration
        & KL\(\downarrow\)
        & \(W_2\downarrow\)
        & MMD\(\downarrow\) \\
      \midrule

      RF & Plain CFG
        & \(\omega=1\)
        & \(0.0044_{0.0012}\) & \(0.0060_{0.0016}\) & \(0.0006_{0.0001}\) \\
      & & \(\omega=1.25\)
        & \(0.0429_{0.0005}\) & \(0.0645_{0.0026}\) & \(0.0053_{0.0003}\) \\
      & & \(\omega=1.5\)
        & \(0.1277_{0.0016}\) & \(0.1939_{0.0052}\) & \(0.0172_{0.0005}\) \\

      & CFG-Zero\(^*\)
        & \(\omega=1\), \(K_{\mathrm{zero}}=1\)
        & \(0.0045_{0.0013}\) & \(0.0062_{0.0017}\) & \(0.0006_{0.0001}\) \\
      & & \(\omega=1\), \(K_{\mathrm{zero}}=2\)
        & \(0.0046_{0.0013}\) & \(0.0064_{0.0017}\) & \(0.0006_{0.0001}\) \\
      & & \(\omega=1.25\), \(K_{\mathrm{zero}}=2\)
        & \(0.0093_{0.0005}\) & \(0.0150_{0.0014}\) & \(0.0015_{0.0002}\) \\

      & CFG-MP
        & \(\omega=3\), \(K_{\mathrm{proj}}=2\)
        & \(1.8740_{0.0116}\) & \(2.9592_{0.0236}\) & \(0.2987_{0.0005}\) \\
      & & \(\omega=2\), \(K_{\mathrm{proj}}=2\)
        & \(1.5907_{0.0101}\) & \(2.5036_{0.0195}\) & \(0.2658_{0.0005}\) \\
      & & \(\omega=1.25\), \(K_{\mathrm{proj}}=2\)
        & \(1.3863_{0.0253}\) & \(2.1654_{0.0297}\) & \(0.2382_{0.0006}\) \\

      & R-CFG++
        & \(\lambda=0.5\), \(\sigma=0\), \(\gamma=1.0\)
        & \(0.0185_{0.0005}\) & \(0.0243_{0.0007}\) & \(0.0019_{0.0001}\) \\
      & & \(\lambda=0.5\), \(\sigma=0.05\), \(\gamma=1.0\)
        & \(0.0186_{0.0005}\) & \(0.0243_{0.0007}\) & \(0.0019_{0.0001}\) \\
      & & \(\lambda=1\), \(\sigma=0\), \(\gamma=1.0\)
        & \(0.0525_{0.0002}\) & \(0.0683_{0.0018}\) & \(0.0055_{0.0002}\) \\
\cmidrule(lr){3-6}
      & PathGuide (Ours)
        & \(\mathrm{seed}_1\)
        & \boldmath{\(0.0038_{0.0011}\)} & \boldmath{\(0.0055_{0.0013}\)} & \boldmath{\(0.0005_{0.0001}\)} \\
      & & \(\mathrm{seed}_2\)
        & \(0.0038_{0.0011}\) & \(0.0055_{0.0012}\) & \(0.0005_{0.0001}\) \\
      & & \(\mathrm{seed}_3\)
        & \(0.0039_{0.0011}\) & \(0.0056_{0.0011}\) & \(0.0005_{0.0001}\) \\

      \midrule

      I-CFM & Plain CFG
        & \(\omega=1\)
        & \(0.0083_{0.0021}\) & \(0.0111_{0.0028}\) & \(0.0012_{0.0003}\) \\
      & & \(\omega=1.25\)
        & \(0.0411_{0.0003}\) & \(0.0634_{0.0019}\) & \(0.0054_{0.0003}\) \\
      & & \(\omega=1.5\)
        & \(0.1200_{0.0009}\) & \(0.1877_{0.0043}\) & \(0.0169_{0.0004}\) \\

      & CFG-Zero\(^*\)
        & \(\omega=1\), \(K_{\mathrm{zero}}=1\)
        & \(0.0083_{0.0021}\) & \(0.0111_{0.0029}\) & \(0.0012_{0.0003}\) \\
      & & \(\omega=1\), \(K_{\mathrm{zero}}=2\)
        & \(0.0083_{0.0021}\) & \(0.0112_{0.0029}\) & \(0.0012_{0.0003}\) \\
      & & \(\omega=1.25\), \(K_{\mathrm{zero}}=2\)
        & \(0.0114_{0.0009}\) & \(0.0182_{0.0015}\) & \(0.0020_{0.0003}\) \\

      & CFG-MP
        & \(\omega=1\), \(K_{\mathrm{proj}}=1\)
        & \(0.1191_{0.0008}\) & \(0.1859_{0.0043}\) & \(0.0167_{0.0004}\) \\
      & & \(\omega=1.25\), \(K_{\mathrm{proj}}=1\)
        & \(0.2131_{0.0017}\) & \(0.3355_{0.0064}\) & \(0.0321_{0.0006}\) \\
      & & \(\omega=1.5\), \(K_{\mathrm{proj}}=1\)
        & \(0.3121_{0.0025}\) & \(0.4944_{0.0081}\) & \(0.0494_{0.0007}\) \\

      & R-CFG++
        & \(\lambda=0.5\), \(\sigma=0\), \(\gamma=1.0\)
        & \(0.0185_{0.0012}\) & \(0.0246_{0.0011}\) & \(0.0023_{0.0002}\) \\
      & & \(\lambda=0.5\), \(\sigma=0.05\), \(\gamma=1.0\)
        & \(0.0186_{0.0012}\) & \(0.0247_{0.0011}\) & \(0.0023_{0.0002}\) \\
      & & \(\lambda=1\), \(\sigma=0\), \(\gamma=1.0\)
        & \(0.0491_{0.0007}\) & \(0.0647_{0.0012}\) & \(0.0056_{0.0002}\) \\
\cmidrule(lr){3-6}
      & PathGuide (Ours)
        & \(\mathrm{seed}_1\)
        & \boldmath{\(0.0073_{0.0018}\)} & \boldmath{\(0.0096_{0.0022}\)} & \boldmath{\(0.0012_{0.0003}\)} \\
      & & \(\mathrm{seed}_2\)
        & \(0.0073_{0.0018}\) & \(0.0097_{0.0021}\) & \(0.0012_{0.0003}\) \\
      & & \(\mathrm{seed}_3\)
        & \(0.0075_{0.0018}\) & \(0.0098_{0.0020}\) & \(0.0012_{0.0003}\) \\

      \midrule

      OT & Plain CFG
        & \(\omega=1\)
        & \(0.0080_{0.0020}\) & \(0.0110_{0.0027}\) & \(0.0012_{0.0003}\) \\
      & & \(\omega=1.25\)
        & \(0.0422_{0.0003}\) & \(0.0657_{0.0021}\) & \(0.0056_{0.0003}\) \\
      & & \(\omega=1.5\)
        & \(0.1217_{0.0010}\) & \(0.1914_{0.0045}\) & \(0.0172_{0.0004}\) \\

      & CFG-Zero\(^*\)
        & \(\omega=1\), \(K_{\mathrm{zero}}=1\)
        & \(0.0080_{0.0020}\) & \(0.0110_{0.0028}\) & \(0.0011_{0.0003}\) \\
      & & \(\omega=1\), \(K_{\mathrm{zero}}=2\)
        & \(0.0080_{0.0020}\) & \(0.0111_{0.0028}\) & \(0.0012_{0.0003}\) \\
      & & \(\omega=1.25\), \(K_{\mathrm{zero}}=2\)
        & \(0.0120_{0.0009}\) & \(0.0195_{0.0017}\) & \(0.0021_{0.0003}\) \\

      & CFG-MP
        & \(\omega=1\), \(K_{\mathrm{proj}}=1\)
        & \(0.1208_{0.0010}\) & \(0.1895_{0.0045}\) & \(0.0170_{0.0004}\) \\
      & & \(\omega=1.25\), \(K_{\mathrm{proj}}=1\)
        & \(0.2146_{0.0019}\) & \(0.3394_{0.0065}\) & \(0.0324_{0.0006}\) \\
      & & \(\omega=1.5\), \(K_{\mathrm{proj}}=1\)
        & \(0.3125_{0.0027}\) & \(0.4977_{0.0083}\) & \(0.0496_{0.0007}\) \\

      & R-CFG++
        & \(\lambda=0.5\), \(\sigma=0\), \(\gamma=1.0\)
        & \(0.0192_{0.0011}\) & \(0.0260_{0.0011}\) & \(0.0024_{0.0002}\) \\
      & & \(\lambda=0.5\), \(\sigma=0.05\), \(\gamma=1.0\)
        & \(0.0192_{0.0011}\) & \(0.0260_{0.0011}\) & \(0.0024_{0.0002}\) \\
      & & \(\lambda=1\), \(\sigma=0\), \(\gamma=1.0\)
        & \(0.0505_{0.0006}\) & \(0.0670_{0.0013}\) & \(0.0059_{0.0002}\) \\
\cmidrule(lr){3-6}
      & PathGuide (Ours)
        & \(\mathrm{seed}_1\)
        & \boldmath{\(0.0073_{0.0017}\)} & \boldmath{\(0.0099_{0.0021}\)} & \boldmath{\(0.0011_{0.0003}\)} \\
      & & \(\mathrm{seed}_2\)
        & \(0.0073_{0.0017}\) & \(0.0099_{0.0020}\) & \(0.0011_{0.0003}\) \\
      & & \(\mathrm{seed}_3\)
        & \(0.0075_{0.0016}\) & \(0.0102_{0.0019}\) & \(0.0012_{0.0003}\) \\

      \midrule

      VP & Plain CFG
        & \(\omega=1\)
        & \(0.0030_{0.0007}\) & \(0.0052_{0.0013}\) & \(0.0005_{0.0001}\) \\
      & & \(\omega=1.25\)
        & \(0.0252_{0.0009}\) & \(0.0369_{0.0027}\) & \(0.0027_{0.0002}\) \\
      & & \(\omega=1.5\)
        & \(0.0905_{0.0018}\) & \(0.1393_{0.0052}\) & \(0.0115_{0.0004}\) \\

      & CFG-Zero\(^*\)
        & \(\omega=1.5\), \(K_{\mathrm{zero}}=1\)
        & \(0.0013_{0.0003}\) & \(0.0019_{0.0005}\) & \(0.0003_{0.0000}\) \\
      & & \(\omega=1.5\), \(K_{\mathrm{zero}}=2\)
        & \(0.0013_{0.0003}\) & \(0.0019_{0.0005}\) & \(0.0003_{0.0000}\) \\
      & & \(\omega=1.25\), \(K_{\mathrm{zero}}=1\)
        & \(0.0019_{0.0004}\) & \(0.0031_{0.0008}\) & \(0.0003_{0.0000}\) \\

      & CFG-MP
        & \(\omega=1\), \(K_{\mathrm{proj}}=2\)
        & \(1.1612_{0.0127}\) & \(1.6279_{0.0133}\) & \(0.2012_{0.0005}\) \\
      & & \(\omega=1.25\), \(K_{\mathrm{proj}}=2\)
        & \(1.2188_{0.0021}\) & \(1.7345_{0.0131}\) & \(0.2098_{0.0005}\) \\
      & & \(\omega=1.5\), \(K_{\mathrm{proj}}=2\)
        & \(1.2585_{0.0019}\) & \(1.8333_{0.0139}\) & \(0.2186_{0.0005}\) \\

      & R-CFG++
        & \(\lambda=0.5\), \(\sigma=0\), \(\gamma=1.0\)
        & \(0.0289_{0.0007}\) & \(0.0387_{0.0023}\) & \(0.0028_{0.0002}\) \\
      & & \(\lambda=0.5\), \(\sigma=0.05\), \(\gamma=1.0\)
        & \(0.0290_{0.0007}\) & \(0.0387_{0.0023}\) & \(0.0028_{0.0002}\) \\
      & & \(\lambda=1\), \(\sigma=0\), \(\gamma=1.0\)
        & \(0.0927_{0.0013}\) & \(0.1313_{0.0045}\) & \(0.0110_{0.0004}\) \\
\cmidrule(lr){3-6}
      & PathGuide (Ours)
        & \(\mathrm{seed}_1\)
        & \(0.0011_{0.0004}\) & \(0.0015_{0.0005}\) & \(0.0003_{0.0000}\) \\
      & & \(\mathrm{seed}_2\)
        & \boldmath{\(0.0011_{0.0004}\)} & \boldmath{\(0.0015_{0.0005}\)} & \boldmath{\(0.0003_{0.0000}\)} \\
      & & \(\mathrm{seed}_3\)
        & \(0.0011_{0.0004}\) & \(0.0015_{0.0005}\) & \(0.0003_{0.0000}\) \\
      \bottomrule
    \end{tabular}
\end{table*}

\begin{table*}[t]
    \caption{
    Full MNIST image-generation sweep results for RF and OT flow variants at \(T=50\).
    FID is computed with \texttt{torch-fidelity==0.4.0}; lower is better.
    All entries are evaluated over three evaluation seeds using \(2^{13}\)
    generated samples per seed, with \(M=2^{11}\) and \(N=2^{12}\).
    Entries report mean \(\pm\) sample standard deviation over three evaluation seeds.
    }
    \label{tab:mnist-baseline-sweeps-t50}
    \vspace{0.5em}
    \centering
    \scriptsize
    \setlength{\tabcolsep}{4pt}
    \renewcommand{\arraystretch}{1.10}
    \begin{tabular}{@{}llp{0.50\textwidth}c@{}}
        \toprule
        Variant & Method & Configuration & FID\(\downarrow\) \\
        \midrule

        RF & Plain CFG
        & \(\omega=1.0\)
        & \(15.802 \pm 0.464\) \\
        & & \(\omega=1.25\)
        & \(16.342 \pm 0.506\) \\
        & & \(\omega=1.5\)
        & \(18.567 \pm 0.640\) \\
        & & \(\omega=2.0\)
        & \(26.028 \pm 0.718\) \\

        & CFG-Zero\(^\ast\)
        & \(\omega=1.0\), \(K_{\mathrm{zero}}=1\)
        & \(29.627 \pm 0.494\) \\
        & & \(\omega=1.25\), \(K_{\mathrm{zero}}=1\)
        & \(26.805 \pm 0.610\) \\
        & & \(\omega=1.5\), \(K_{\mathrm{zero}}=1\)
        & \(26.534 \pm 0.747\) \\
        & & \(\omega=2.0\), \(K_{\mathrm{zero}}=1\)
        & \(30.908 \pm 0.844\) \\

        & CFG-MP
        & \(\omega=1.0\), \(K_{\mathrm{proj}}=3\)
        & \(30.951 \pm 0.999\) \\
        & & \(\omega=1.25\), \(K_{\mathrm{proj}}=3\)
        & \(37.156 \pm 0.869\) \\
        & & \(\omega=1.5\), \(K_{\mathrm{proj}}=3\)
        & \(43.712 \pm 0.621\) \\
        & & \(\omega=2.0\), \(K_{\mathrm{proj}}=3\)
        & \(57.352 \pm 0.725\) \\

        & R-CFG++
        & \(\lambda_{\max}=1.0\), \(\gamma=1.0\), \(\sigma=0.0\)
        & \(16.979 \pm 0.473\) \\
\cmidrule(lr){3-4}
        & PathGuide (Ours)
        & \(\mathrm{seed}_1\)
        & \(\mathbf{15.699 \pm 0.437}\) \\
        & & \(\mathrm{seed}_2\)
        & \(15.737 \pm 0.446\) \\
        & & \(\mathrm{seed}_3\)
        & \(15.720 \pm 0.430\) \\

        \midrule

        OT & Plain CFG
        & \(\omega=1.0\)
        & \(7.578 \pm 0.263\) \\
        & & \(\omega=1.25\)
        & \(8.262 \pm 0.254\) \\
        & & \(\omega=1.5\)
        & \(10.596 \pm 0.215\) \\
        & & \(\omega=2.0\)
        & \(18.230 \pm 0.059\) \\

        & CFG-Zero\(^\ast\)
        & \(\omega=1.0\), \(K_{\mathrm{zero}}=1\)
        & \(11.918 \pm 0.298\) \\
        & & \(\omega=1.25\), \(K_{\mathrm{zero}}=1\)
        & \(11.327 \pm 0.261\) \\
        & & \(\omega=1.5\), \(K_{\mathrm{zero}}=1\)
        & \(12.666 \pm 0.303\) \\
        & & \(\omega=2.0\), \(K_{\mathrm{zero}}=1\)
        & \(18.698 \pm 0.032\) \\

        & CFG-MP
        & \(\omega=1.0\), \(K_{\mathrm{proj}}=3\)
        & \(24.475 \pm 0.128\) \\
        & & \(\omega=1.25\), \(K_{\mathrm{proj}}=3\)
        & \(28.995 \pm 0.130\) \\
        & & \(\omega=1.5\), \(K_{\mathrm{proj}}=3\)
        & \(33.313 \pm 0.109\) \\
        & & \(\omega=2.0\), \(K_{\mathrm{proj}}=3\)
        & \(43.200 \pm 0.302\) \\

        & R-CFG++
        & \(\lambda_{\max}=1.0\), \(\gamma=1.0\), \(\sigma=0.0\)
        & \(15.038 \pm 0.356\) \\
\cmidrule(lr){3-4}
        & PathGuide (Ours)
        & \(\mathrm{seed}_1\)
        & \(7.403 \pm 0.281\) \\
        & & \(\mathrm{seed}_2\)
        & \(7.384 \pm 0.282\) \\
        & & \(\mathrm{seed}_3\)
        & \(\mathbf{7.375 \pm 0.285}\) \\

        \bottomrule
    \end{tabular}
\end{table*}

\begin{table*}[t]
    \caption{
    Endpoint path discrepancy at \(t_i \approx 1\) for \(T=20\), our method is evaluated online where \(M=N=2^{11}\) against the plain conditional learned field under flow type is VP.
    We report mean \(\pm\) sample standard deviation. Lower is better for all metrics.
    }
    \vspace{0.5em}
    \label{tab:online-path-discrepancy-full}
    \centering
    \small
    \setlength{\tabcolsep}{12pt}
    \renewcommand{\arraystretch}{1.08}
    \begin{tabular}{lccc}
        \toprule
        Method
        & \(\mathrm{KL}\downarrow\)
        & \(\mathrm{MMD}\text{-}\mathrm{RBF}^2\downarrow\)
        & \(\mathrm{W}_2\downarrow\) \\
        \midrule
        \(\omega\equiv 1\)
        & \(0.01261 \pm 0.00184\)
        & \(0.00271 \pm 0.00090\)
        & \(0.01755 \pm 0.00420\) \\
        Practical schedule
        & \(0.01232 \pm 0.00216\)
        & \(0.00260 \pm 0.00073\)
        & \(0.01569 \pm 0.00340\) \\
        Exact field \(u\)
        & \(0.01147 \pm 0.00127\)
        & \(0.00258 \pm 0.00059\)
        & \(0.01516 \pm 0.00309\) \\
        \bottomrule
    \end{tabular}
\end{table*}

\section*{NeurIPS Paper Checklist}

\begin{enumerate}

\item {\bf Claims}
    \item[] Question: Do the main claims made in the abstract and introduction accurately reflect the paper's contributions and scope?
    \item[] Answer: \answerYes{}
    \item[] Justification: The abstract and introduction accurately state the paper's scope and contributions.

\item {\bf Limitations}
    \item[] Question: Does the paper discuss the limitations of the work performed by the authors?
    \item[] Answer: \answerYes{}
    \item[] Justification: Section~\ref{sec:conclusions} discusses the main limitations of the method and experiments.

\item {\bf Theory assumptions and proofs}
    \item[] Question: For each theoretical result, does the paper provide the full set of assumptions and a complete (and correct) proof?
    \item[] Answer: \answerYes{}
    \item[] Justification: The theoretical claims has a proof sketch, complete proofs are given in Appendix~\ref{app:proofs}, assumptions stated in the relevant statements and collected in Appendix~\ref{app:assumptions}.

    \item {\bf Experimental result reproducibility}
    \item[] Question: Does the paper fully disclose all the information needed to reproduce the main experimental results of the paper to the extent that it affects the main claims and/or conclusions of the paper (regardless of whether the code and data are provided or not)?
    \item[] Answer: \answerYes{}
    \item[] Justification: Appendix~\ref{app:exp-details} provides the experimental protocol and details needed to reproduce the main experimental results.

\item {\bf Open access to data and code}
    \item[] Question: Does the paper provide open access to the data and code, with sufficient instructions to faithfully reproduce the main experimental results, as described in supplemental material?
    \item[] Answer: \answerYes{}
    \item[] Justification: The supplementary material includes code, configurations, checkpoints, and reproduction commands for the main experiments.

\item {\bf Experimental setting/details}
    \item[] Question: Does the paper specify all the training and test details (e.g., data splits, hyperparameters, how they were chosen, type of optimizer) necessary to understand the results?
    \item[] Answer: \answerYes{}
    \item[] Justification: Appendix~\ref{app:exp-details} specifies all details of experiment, including training, evaluation, hyperparameters, and how they where chosen.

\item {\bf Experiment statistical significance}
    \item[] Question: Does the paper report error bars suitably and correctly defined or other appropriate information about the statistical significance of the experiments?
    \item[] Answer: \answerYes{}
    \item[] Justification: The reported tables and curves include mean and sample-standard-deviation results over matched inference seeds, full details in Appendix~\ref{app:exp-details}.

\item {\bf Experiments compute resources}
    \item[] Question: For each experiment, does the paper provide sufficient information on the computer resources (type of compute workers, memory, time of execution) needed to reproduce the experiments?
    \item[] Answer: \answerYes{}
    \item[] Justification: Appendix~\ref{app:computational-cost} reports the hardware, software environment, wall-clock times, and total compute.
    
\item {\bf Code of ethics}
    \item[] Question: Does the research conducted in the paper conform, in every respect, with the NeurIPS Code of Ethics \url{https://neurips.cc/public/EthicsGuidelines}?
    \item[] Answer: \answerYes{}
    \item[] Justification: The research conforms to the NeurIPS Code of Ethics.

\item {\bf Broader impacts}
    \item[] Question: Does the paper discuss both potential positive societal impacts and negative societal impacts of the work performed?
    \item[] Answer: \answerNA{}
    \item[] Justification: The work is a methodological study of guidance selection and does not introduce new generative capabilities beyond the underlying models.
    
\item {\bf Safeguards}
    \item[] Question: Does the paper describe safeguards that have been put in place for responsible release of data or models that have a high risk for misuse (e.g., pre-trained language models, image generators, or scraped datasets)?
    \item[] Answer: \answerNA{}
    \item[] Justification: The release does not include high-risk pretrained models or scraped datasets.

\item {\bf Licenses for existing assets}
    \item[] Question: Are the creators or original owners of assets (e.g., code, data, models), used in the paper, properly credited and are the license and terms of use explicitly mentioned and properly respected?
    \item[] Answer: \answerYes{}
    \item[] Justification: Existing code and data assets are credited, and their licenses and terms are summarized in Appendix~\ref{app:licenses}.

\item {\bf New assets}
    \item[] Question: Are new assets introduced in the paper well documented and is the documentation provided alongside the assets?
    \item[] Answer: \answerYes{}
    \item[] Justification: The code and checkpoints are documented with a README, training commands, and license information.

\item {\bf Crowdsourcing and research with human subjects}
    \item[] Question: For crowdsourcing experiments and research with human subjects, does the paper include the full text of instructions given to participants and screenshots, if applicable, as well as details about compensation (if any)? 
    \item[] Answer: \answerNA{}
    \item[] Justification: This work does not involve crowdsourcing or human subjects.

\item {\bf Institutional review board (IRB) approvals or equivalent for research with human subjects}
    \item[] Question: Does the paper describe potential risks incurred by study participants, whether such risks were disclosed to the subjects, and whether Institutional Review Board (IRB) approvals (or an equivalent approval/review based on the requirements of your country or institution) were obtained?
    \item[] Answer: \answerNA{}
    \item[] Justification: This work does not involve human subjects.

\item {\bf Declaration of LLM usage}
    \item[] Question: Does the paper describe the usage of LLMs if it is an important, original, or non-standard component of the core methods in this research? Note that if the LLM is used only for writing, editing, or formatting purposes and does \emph{not} impact the core methodology, scientific rigor, or originality of the research, declaration is not required.
    \item[] Answer: \answerNA{}
    \item[] Justification: LLMs were not used as a component of the core research methodology.

\end{enumerate}

\end{document}